\documentclass[11pt]{article} 

\usepackage[utf8]{inputenc} 

\usepackage{wrapfig}

\usepackage[letterpaper]{geometry}

\usepackage{graphicx} 

\usepackage{booktabs} 
\usepackage{array} 
\usepackage{paralist} 
\usepackage{verbatim} 
\usepackage{mathrsfs}
\usepackage{amssymb}
\usepackage{amsthm}
\usepackage{amsmath,amsfonts,amssymb,esint}
\usepackage{graphics}
\usepackage{enumerate}
\usepackage{mathtools}
\usepackage{xfrac}
\usepackage{bbm}
\usepackage{subcaption}
\usepackage{times}

\usepackage{smile}

\usepackage[usenames,dvipsnames]{xcolor}
\usepackage[colorlinks=true, pdfstartview=FitV, linkcolor=blue, citecolor=blue, urlcolor=blue]{hyperref}

\numberwithin{equation}{section}

\newtheorem{theorem}{Theorem}[section]

\newtheorem{proposition}[theorem]{Proposition}
\newtheorem{lemma}[theorem]{Lemma}
\newtheorem{definition}[theorem]{Definition}
\theoremstyle{definition}
\newtheorem{remark}[theorem]{Remark}

\newcommand*{\supp}{\ensuremath{\mathrm{supp\,}}}

\newcommand*{\R}{\ensuremath{\mathbb{R}}}
\newcommand*{\E}{\ensuremath{\mathbb{E}}}

\renewcommand*{\tilde}{\widetilde}
\renewcommand*{\hat}{\widehat}

\newcommand{\red}[1]{\textcolor{black}{#1}}

\newcommand{\sF}{\mathsf{F}}

\newcommand{\cL}{\mathcal{L}}

\everymath{\displaystyle}

\usepackage[nottoc,notlot,notlof]{tocbibind}
\allowdisplaybreaks

\begin{document}

\title{Deep Neural Operator Learning for Probabilistic Models}

\author{Erhan
Bayraktar \thanks{\noindent Department of
Mathematics, University of Michigan, Ann Arbor, 48109; email: erhan@umich.edu. This author is supported in part \red{by} the Susan M. Smith Professorship and in part \red{by} the National Science Foundation under grant \#DMS-2507940.} ~~~Qi Feng\thanks{\noindent Department of
Mathematics, Florida State University, Tallahassee, 32306; email: qfeng2@fsu.edu. This author is partially supported by the National Science Foundation under grant \#DMS-2420029.}
 ~~~  Zecheng Zhang\thanks{\noindent Department of Applied Computational Mathematics and Statistics, University of Notre Dame, Notre Dame 46556; zzhang48@nd.edu. This author is partially supported by the Department of Energy DE-SC0025440}~~~  Zhaoyu Zhang\thanks{ \noindent Department of Mathematical Sciences, DePaul University,
Chicago, IL, USA; email: zzhang97@depaul.edu. }
}
\maketitle

\begin{abstract}
We propose a deep neural-operator framework for a general class of \red{probabilistic} models. Under global Lipschitz conditions on the operator over the entire Euclidean space—and for a broad class of probabilistic models—we establish a universal approximation theorem with explicit network-size bounds for the proposed architecture. The underlying stochastic processes are required only to satisfy integrability and general tail-probability conditions. We verify these assumptions for both European and American option-pricing problems within the forward–backward SDE (FBSDE) framework, which in turn covers a broad class of operators arising from parabolic PDEs, with or without free boundaries. Finally, we present a numerical example for a basket of American options, demonstrating that the learned model produces optimal stopping boundaries for new strike prices without retraining.
\end{abstract}

\section{Introduction}
Operator learning \cite{chen1995universal, lin2023b, lu2021learning, li2020fourier, li2023fourier, wang2025laplacian} uses deep neural networks to approximate mappings between functions or function spaces, enabling efficient solutions to a wide range of computational science problems.
For instance, it can learn the mapping from the initial condition of a partial differential equation (PDE) to its corresponding solution.
Another example is to learn the mapping from fine-scale PDE solutions to coarse-scale ones \cite{howard2022multifidelity, leung2022nh}, effectively performing model upscaling.
A key advantage of operator learning is its ability to handle parametric PDEs.
For example, when the PDE initial condition is parameterized by free variables, the operator learning framework can learn the mapping between the space formed by all initial conditions and the space of their corresponding \red{solutions}.
Once trained, a deep neural operator (DNO) can instantly predict the solution for any new initial condition within the same function space.
Compared to \red{standard} numerical solvers, operator learning offers much faster and more cost-efficient computations.
This approach has been widely applied to inverse PDE problems, where it is often integrated with standard numerical solvers—either providing an initial approximation refined by numerical methods or serving as a rapid surrogate model to accelerate the overall solution process.

Many popular deep neural operators (DNOs) have been proposed \cite{chen1995universal, lin2023b, lu2021learning, li2020fourier, li2023fourier, wang2025laplacian}, and operator learning has been successfully applied to solving practical problems \cite{pathak2022fourcastnet, bodnar2025foundation}, making it an important machine learning framework for large-scale computational applications.
Theoretically, the first operator learning framework was proposed in the seminal works \cite{chen1995universal, chen1993approximations}, where the authors introduced a shallow universal approximation architecture for nonlinear operators. This foundational theory has directly inspired the design of several modern DNOs, such as DeepONet \cite{lu2021learning}.
Later, operator learning—viewed as a mapping between infinite-dimensional function spaces—\red{was} analyzed in \cite{bhattacharya2021model, liu2024deep}, where the approximation error was quantified with respect to the discretization size of the input function, network complexity, and related parameters.
In \cite{kovachki2021universal, kovachki2023neural}, the authors generalized the framework of several neural operators, including the Fourier Neural Operator (FNO), and analyzed their universal approximation properties in shallow network settings, though without establishing convergence rates.
In \cite{lanthaler2022error, liu2024neural}, the convergence rate of DeepONet was established for a class of PDE operators.
In \cite{lanthaler2025parametric}, the author studied the lower bound of the convergence rate for PCA-Net, with potential generalizations to other DNOs, while an upper bound result for PCA-Net was given in \cite{lanthaler2023operator}.
Across the existing literature, establishing rigorous convergence rates for general operators—without restricting to specific PDE formulations—remains one of the central theoretical challenges in operator learning.
One notable contribution in this direction is provided by \cite{liu2024neural}, which unifies various formulations of neural operators and rigorously establishes convergence rates as the network depth and width increase, for general operators not tied to any specific PDE.
Specifically, the total number of trainable parameters \red{needed} to reach \red{an} $\varepsilon$ error in \red{the} $L^{\infty}$ norm \red{scales} as $\varepsilon^{-\varepsilon^{-d_2}}$.

However, all the aforementioned studies address deterministic problems on bounded domains, without involving any stochastic processes or probabilistic models.
To the best of our knowledge, the neural operator approach has only recently been extended to Forward–Backward Stochastic Differential Equations (FBSDEs) \cite{furuya2024simultaneously}, which can be applied to solve European-type option pricing problems, and to Dynamic Stackelberg Games \cite{alvarez2024neural}. In this paper, we develop a neural-operator framework under general Lipschitz conditions for broad classes of stochastic processes satisfying integrability and tail-probability assumptions. {\color{black}Our method extends the branch--trunk network framework of \cite{liu2024neural} from bounded domains to the whole space $\mathbb R^{d_1}$ by controlling the tail probability of the underlying stochastic process.} In particular, we employ our neural operator to address the American option pricing problem and its associated optimal stopping boundary problem. Recently, American option pricing and optimal stopping problems have been investigated using deep neural networks in \cite{sirignano2018dgm, lapeyre2021neural, hure2020deep, becker2019deep, becker2021solving, reppen2022neural, bayer2023optimal, gao2023convergence,bayraktar2024deep, gonon2024deep,herrera2024optimal}. Theoretical results on the continuity property of \red{the} optimal stopping boundary have been investigated in \cite{soner2025stopping}, while a more general Stefan-type problem for partial differential equations (PDEs) with free boundaries has been studied in \cite{shkolnikov2024deep}. In these existing works, the methods are typically designed for a fixed terminal payoff function, requiring retraining of the network when the terminal function changes. In contrast, by adopting the operator learning perspective, our trained model can directly generate the optimal stopping boundary for a new terminal payoff function without retraining. Our general neural-operator approximation results encompass both European and American options in the FBSDE setting. Through the Feynman–Kac correspondence (and its optimal-stopping/variational-inequality version for American options), the same guarantees apply to the corresponding PDEs and free-boundary PDEs.

The paper is organized as follows. Section \ref{sec: assm} states the standing assumptions on the underlying probabilistic models and the conditions imposed on the operators. Section \ref{sec: operator} presents the deep neural operator architecture, including its construction and size bounds. Section \ref{sec: universal} establishes a universal approximation theorem for functions, functionals, and operators under these assumptions. Section \ref{sec: pricing european} and Section \ref{sec: pricing american} \red{verify} that European- and American-style option pricing problems within the FBSDE framework satisfy the assumptions; by the Feynman–Kac representation (and its variational-inequality form for optimal stopping), the theorem then applies to the associated parabolic PDEs, with or without free boundaries. Section \ref{sec: numerics} provides a numerical example for a basket of American options, demonstrating that the learned model produces optimal stopping boundaries for new strike prices without retraining.

\section{Assumptions}\label{sec: assm}
Let
\(
(\Omega, \mathcal F, \{\mathcal F_t\}_{t \in [0,T]}, \mathbb P)
\)
be a filtered probability space satisfying the usual conditions and right-continuity. In this paper, we denote \red{by}
\[
X = (X_t)_{t \in [0,T]}
\]
\red{with} $X_0 = x$ an \(\mathbb R^{d_1}\)-valued adapted process \red{that is} progressively measurable with respect to \(\{\mathcal F_t\}_{t \in [0,T]}\).
\begin{definition}
    For $x\in \mathbb{R}^d$, $|x| := \sqrt{\sum_{i=1}^d x_i^2}$, and $\|x\|_{\infty} := \max_{1\le i \le d} |x_i|. $
\end{definition}

\begin{assumption}\label{assumption_X: power}
For any $p>0$, there exists a constant $C_p>0$, such that
\begin{align}
\mathbb{E}\left[\sup_{0 \leq t \leq T} | X_t|^p\right] \leq C_p.
\end{align}
\end{assumption}

\begin{assumption}\label{assumption_X: tail}
 There \red{exist} a constant $C_T$ depending on \red{the} time \red{horizon} $T$\red{, and constants $c>0$ and $\alpha>0$,} such that for any $r >0$,
\begin{align}\mathbb P\left(\sup_{t \in [0,T]} |X_t - x| \geq r\right) \leq \exp\left(-\frac{c r^\alpha}{C_T}\right).
\end{align}
\end{assumption}
\begin{definition}\label{def: omega r}
  For any $r>0$, we define
\begin{align}
    \Omega_r := \{x\in \mathbb R^{d_1} : |x| \le r \}\quad \text{and}\quad \Omega_r^C:= \mathbb{R}^{d_1} \setminus \Omega_r,
\end{align}
    \red{and} we define the cube correspondingly as follows
\begin{equation*}
Q_r := [-r,r]^{d_1} \;=\; \left\{ x\in\mathbb{R}^{d_1} : \|x\|_{\infty} \leq r \right\}.
\end{equation*}
\end{definition}

\begin{definition}[Input space \(\cG \)]
Define the input space as below
\[
\cG
:=
\left\{
  g : \mathbb R^{d_1} \to \mathbb R
  \ \Big|\
  g(X_t) \text{ is progressively measurable and }
  \mathbb E\!\left[\sup_{0 \le s \le T} |g(X_s)|^2\right] < \infty
\right\}.
\]
The space \(\cG \) is equipped with the norm
\[
\|g\|_{\mathcal S^2}
:=
\left(
  \mathbb E\!\left[\sup_{0 \le s \le T} |g(X_s)|^2\right]
\right)^{1/2}.
\]
\end{definition}


\begin{definition}[Output space \(\mathcal U\)]
Define the output space as below
\[
\mathcal U
:=
\left\{
  u : [0,T] \times \mathbb R^{d_2} \to \mathbb R
  \ \Big|\
  u(t,X_t) \text{ is progressively measurable and }
  \mathbb E\!\left[\sup_{0 \le s \le T} |u(s,X_s)|^2\right] < \infty
\right\},
\]
with the norm
\begin{align}
\|u\|_{\mathcal S^2}
:=
\left(
  \mathbb E\!\left[\sup_{0 \le s \le T} |u(s,X_s)|^2\right]
\right)^{1/2}.
\end{align}
\end{definition}
\begin{remark}\label{Xt notation}
    We \red{write} $X=(X_t)_{t\in[0,T]}$ for a generic stochastic process. Depending on the context, $X\in \text{Domain}(\mathcal G)$ implies $X_t\in \mathbb R^{d_1}$, while  $X\in \text{Domain}(\mathcal U)$ implies $X_t\in \mathbb R^{d_2}$.
\end{remark}

We impose the following polynomial growth condition and Lipschitz condition.

\begin{assumption}
    \label{assumption: polynomial}For any function $g\in \cG $, and $x\in \mathbb R^{d_1}$, \red{there exists} a constant $C_g$, such that
\begin{align}
    \red{|g(x)|}\le C_g(1+|x|^p), \quad \text{for} \red{\text{ some }} p>0.
\end{align}
\end{assumption}

\begin{assumption}\label{assumption_input}
Any function $g\in \cG $ is Lipschitz with a Lipschitz constant no more than $L_g>0$: $$|g(x_1) - g(x_2)|\leq L_{g} \red{|x_1 -x_2|}$$ for any $x_1,x_2\in \mathbb{R}^{d_1}$.

\end{assumption}

\begin{assumption}\label{assumption_U}
Any function $u \in \mathcal U$ is Lipschitz with a Lipschitz constant no more than $L_u>0$: $$\red{|u(t,x_1) - u(t,x_2)|}\leq L_{u}|x_1 -x_2|,$$ for any \red{$t\in[0,T]$ and} $x_1,x_2\in \mathbb{R}^{d_2}$.

\end{assumption}

\begin{assumption}\label{assumption_G}
    Assume the operator
\[
\Gamma : \cG  \longrightarrow \mathcal U,
\qquad g \longmapsto u = \Gamma (g),
\] from $\cG$ to $\cU$ is Lipschitz\red{, i.e.,} there exists $L_\Gamma\red{>0}$ such that for any $g_1,g_2\in \cG$, we have
    \[
\mathbb E\!\left[
  \sup_{0 \le t \le T}
  |\Gamma (g_1) (X_t) - \Gamma (g_2)(X_t)|^2
\right]
\le
L_\Gamma^2 \,
\mathbb E\!\left[
  \sup_{0 \le t \le T}
  |g_1(X_t) - g_2(X_t)|^2
\right],
\ \forall g_1,g_2 \in \cG .
\]
or equivalently,
\[
\|\Gamma(g_1) - \Gamma(g_2)\|_{S^2} \leq L_\Gamma \|g_1 - g_2\|_{S^2}.
\]
As mentioned in Remark \ref{Xt notation}, the process $X \in $ Domain($\mathcal U$), i.e. $X_t\in \mathbb R^{d_2}$.
\end{assumption}

\begin{definition}[Lipschitz functional]
\label{def.lip_functional}
    We say a functional $\sF: \cV \rightarrow \R$, where $\cV $ could either be the input space $\cG$ or the output space $\cU$\red{,} is Lipschitz with Lipschitz constant $L_\sF$ if
    $$
    |\sF(v_1)-\sF(v_2)|\leq L_\sF\|v_1-v_2\|_{S^2}, \ \forall v_1,v_2 \in \cV.
    $$
\end{definition}

\begin{lemma}[Theorem 13.7(ii) of \cite{tu2011manifolds}]\label{lemma_pou}
        Let $\{\Omega_k\}_{k=1}^M$ be an open cover of a compact smooth manifold $\mathcal{M}$\red{.} There exists a $C^{\infty}$ partition of unity $\{\omega_k\}_{k=1}^M$ \red{subordinate} to $\{\Omega_k\}_{k=1}^M$ such that $\red{\supp(\omega_k)}\subset \Omega_k$ for any $k$.
    \end{lemma}

\section{Deep Operator Learning Architecture }\label{sec: operator}

Operator learning aims to approximate mappings between infinite-dimensional function spaces, distinguishing itself from traditional neural networks, which approximate functions directly. Specifically, given a nonlinear operator \(\Gamma\) that maps an input function \( g \) to an output function \(\Gamma(g)\), the objective is to learn \(\Gamma\) using a neural network architecture.
In this work, the nonlinear operator \(\Gamma\) represents a pricing operator, and the goal is to approximate it via a neural network-based approach.
\red{We adopt the notation of \cite{liu2024deep} and define a feedforward ReLU network} $ {q}: \mathbb{R}^{d_1}\rightarrow \mathbb{R} $ as
\begin{align}
	q(\xb)=W_L\cdot \text{ReLU}\left( W_{L-1}\cdots \text{ReLU}(W_1 x+b_1)+ \cdots +b_{L-1}\right)+b_L,
	\label{eqn_relu_net}
\end{align}
where \red{the} $W_l$'s are weight matrices, \red{the} $b_l$'s are bias vectors, $\text{ReLU}(a)=\max\{a,0\}$ is the rectified linear unit \red{(ReLU) activation} applied element-wise, and $\Omega$ is the domain.
We define the network class $\mathcal{F}_{\rm NN}: \mathbb{R}^{d_1} \rightarrow \mathbb{R}^{d_2}:$
\begin{align}
	\cF_{\rm NN}(d_1, d_2, \mathcal{L}, \mathfrak{p}, K, \kappa, R)=\{&[q_1, q_2,...,q_{d_2}]^{\intercal}\in\mathbb{R}^{d_2}: \mbox{ for each }k=1,...,d_2,\nonumber\\
	&q_k:\mathbb{R}^{d_1}\rightarrow \mathbb{R} \mbox{ is in the form of (\ref{eqn_relu_net}) with } \mathcal{L} \mbox{ layers, width bounded by } \mathfrak{p}, \nonumber\\
	& \|q_k\|_{L^{\infty}}\leq R, \ \|W_l\|_{\infty,\infty}\leq \kappa, \ \|b_l\|_{\infty}\leq \kappa,\  \sum_{l=1}^L \|W_l\|_0+\|b_l\|_0\leq K, \ \forall l   \},
	\label{eq.FNN}
\end{align}
where
$
\|q\|_{L^{\infty}(\Omega)}=\sup\limits_{\xb\in \Omega} |q(\xb)|,\ \|W_l\|_{\infty,\infty}=\max\limits_{i,j} |W_{i,j}|,\ \|b\|_{\infty}=\max\limits_{i} |b_i|
$,
and $\|\cdot\|_0$ denotes the number of nonzero elements of its argument.
The network class above has input dimension $d_1$, output dimension $d_2$, $\mathcal{L}$ layers, width $\mathfrak{p}$, and \red{at most $K$ nonzero parameters}.
All parameters are bounded by $\kappa$, and each element in the output is bounded by $R$.

The objective of operator learning is to construct an operator network \(\Gamma_{p}[\cdot; \theta]\) that approximates \(\Gamma\).
One approximation structure \cite{chen1995universal, lu2021learning} which is widely adopted in designing DNOs is the following.
To better demonstrate the idea of the approximation, we denote \red{by} $y$ the independent variable of the output function of the operator $\Gamma$, and denote \red{by} $\Gamma(g; \theta)(y)$ the neural operator approximation to $\Gamma$\red{. We} then have the approximation,
\begin{equation}
    \Gamma(g)(y) \approx \Gamma(g; \theta)(y) :=  \sum_{k = 1}^{N_1} \widetilde{a}_k(\bg; \hat{\theta} )\widetilde{q}_k(y; \tilde{\theta})
\end{equation}
where $g$ is the input function with a discretization denoted as $\bg$, $\Gamma(g)$ is the output function\red{, which} depends on $y$, and \(\theta = \{\hat{\theta}, \tilde{\theta}\}\) represents the set of trainable parameters of the operator network \red{consisting} of $\tilde{a}_k$ and $\tilde{q}_k$.
The framework was first proposed in \cite{chen1995universal, chen1993approximations} as a two-layer universal approximation scheme for nonlinear operators.
It was later extended computationally to deep neural network architectures in \cite{lu2021learning}, and was finally rigorously analyzed in terms of error convergence and generalization in \cite{liu2024neural}.
Following the terminology widely adopted in recent literature, we use the notation introduced in \cite{lu2021learning} to name the \red{substructures} of the network.
Specifically, the architecture consists of two components:
\begin{itemize}
    \item \textbf{Branch network}: \( \tilde{a}(\gb; \hat{\theta}) = (\tilde{a}_1(\gb; \hat{\theta}_1), \dots, \tilde{a}_{N_1}(\gb; \hat{\theta}_{N_1}))^T\) \red{encodes} the input function $g$ and the operator $\Gamma$.
    Each component \(\tilde{a}_k: \mathbb{R}^{{N_2}}\rightarrow \mathbb{R}\) is implemented as a sum of fully connected neural networks.
Specifically, $\tilde{a}_k = \sum_{n = 1}^{H}a_n\tilde{b}_n$, where $N_2$ is the size of the discretization $\gb$\red{,} $H$ is the number of basis \red{functions, each} represented as a network $\tilde{b}_n$ in \red{the} neural network class $\cF_2=\cF_{\rm NN}(N_2,1,\cL_2,\mathfrak p_2,K_2,\kappa_2)$, \red{and the} $a_n$ are constants.
The size of the entire network $\beta$ (containing all $\tilde{b}_k$) is specified in Theorem \ref{thm_operator}.
    \item \textbf{Trunk network}: $\tilde{q}(y; \tilde{\theta}) = (\tilde{q}_1, ..., \tilde{q}_{N_1})$.
    Here each $\tilde{q}_k:\mathbb{R}^{d_2}\rightarrow \mathbb R$ is a network in \red{the} class $\cF_1=\cF_{\rm NN}(d_2,1,\cL_1,\mathfrak p_1,K_1,\kappa_1)$ with size to be specified in Theorem \ref{thm_operator}.
\end{itemize}


 The following lemma from \red{\cite[Proposition 3]{yarotsky2017error}} shows that \red{the product function} can be approximated by a network with arbitrary accuracy.
\begin{lemma}
\label{lemma_pp3}
    Given $M>0$ and $\varepsilon>0$, there is a ReLU network $\widetilde{\times}: \mathbb{R}^2\rightarrow \mathbb{R}$ in $\cF_{\rm NN}(2,1, \mathcal{L}, \mathfrak{p}, K, \kappa, R)$ such that for any $|x|\leq M, |y|\leq M$, we have
   $$|\widetilde{\times}(x,y) -xy|<\varepsilon.$$
The network architecture has
\begin{align}
    \mathcal{L} = O(\log \varepsilon^{-1}),\ \mathfrak{p} = 6,\ K=O(\log \varepsilon^{-1}),  \ \kappa=O(\varepsilon^{-1}), \ R=M^2.
\end{align}
The constant hidden in $O$ depends on $M$.
\end{lemma}

\section{Neural scaling of operator learning}\label{sec: universal}

\subsection{Function Approximation}
We will first prove the function approximation \red{result} and establish the convergence rate.
The results will be used in proving the operator approximation Theorem \ref{thm_operator}.
\begin{theorem}\label{thm_function}
Assume Assumptions \ref{assumption_X: power}, \ref{assumption_X: tail}, \ref{assumption: polynomial}, \ref{assumption_input} hold.
For any $\varepsilon>0$, set \
$r = \left\lceil -\frac{2C_T}{c}\log \frac{\varepsilon}{4C_0}\right\rceil^{\frac{1}{\alpha}}$ in Assumption \ref{assumption_X: tail},
and set \textcolor{black}{$N = C\sqrt{d_1}\varepsilon^{-1/2}(\log (\varepsilon^{-1}))^{\frac{1}{\alpha}}$, with \red{a} constant $C$ \red{depending} on $C_T, C_0, c, L_g$.}
Let $\{\cbb_k\}_{k=1}^{N^{d_1}}$ be a uniform grid on $Q_r$ (covering $\Omega_r$) with spacing $\red{2r/(N-1)}$ along each dimension.
There exists a network architecture $\cF_{\rm NN}(d_1, 1, \mathcal{L}, \mathfrak{p}, K, \kappa, R)$ and networks $\{\widetilde{q}_k\}_{k=1}^{N^{d_1}}$ with $\widetilde{q}_k\in \cF_{\rm NN}(d_1, 1, \cL, \mathfrak p, K, \kappa, R)$ for $k=1,...,N^{d_1}$, such that for any $g\in \mathcal G$, we have
    \begin{align}
    \mathbb E \Big[ \sup_{0\le t\le T}\Big|g(X_t)-\sum_{k=1}^{N^{d_1}} g(\cbb_k)\widetilde{q}_k(X_t)\Big|^2\Big]\leq \varepsilon.
\label{equation_thm3_not_in_functional_form}
    \end{align}
Such a network architecture has
    \begin{align*}
&\textcolor{black}{\mathcal{L} = O\left(d_1^2\log d_1 + \frac{d_1^2+d_1}{2} \log(\varepsilon^{-1}) +\frac{d_1^2+d_1 p}{\alpha}\log \log (\varepsilon^{-1} )\right)}, \mathfrak p = O(1), \\
& K = O\left(d_1^2\log d_1 + \frac{d_1^2+d_1}{2} \log(\varepsilon^{-1}) +\frac{d_1^2+d_1p}{\alpha}\log \log (\varepsilon^{-1} )\right), \\
& \textcolor{black}{\kappa=O(d_1^{-\frac{d_1}{2}}\varepsilon^{-\frac{d_1+1}{2}} (\log(\varepsilon^{-1}))^{\frac{d_1+p}{\alpha}}) }, R = 1.
    \end{align*}
Here, the network order is determined by the constants $C_T, C_0, c, \alpha, L_g$ specified in Assumptions \ref{assumption_X: power}, \ref{assumption_X: tail}, \ref{assumption: polynomial}, and \ref{assumption_input}, which concern the function $g$ and \red{the} process $\{X_t\}_{0\le t\le T}$. In particular, we denote \red{by} $C_0$ \red{an} upper bound of $\mathbb E\Big[\sup_{0\le t\le T} \Big|1+|X_t|^p \Big|^4 \Big]^{1/2}$.
\end{theorem}

\begin{proof}
Recall $\Omega_r$ and $\Omega_r^C$ from Definition \ref{def: omega r}. Without loss of generality, we assume that the origin $0\in \Omega_r$. We then decompose $g(X_t)=g(X_t)\mathbf 1_{\Omega_r}\red{(X_t)}+g(X_t)\mathbf 1_{\Omega_r^C}\red{(X_t)}$\red{.}
For the compact domain $\Omega_r\subset Q_r$, we \red{partition $Q_r$ into} $N^{d_1}$ subcubes for some $N$ to be specified later. We first approximate $g(X_t)\mathbf 1_{\Omega_r}\red{(X_t)}$ on each cube by a constant function and then assemble them together to get an approximation of $\{g(X_t)\}_{0\le t\le T}$ on $\Omega_r$. Denote the centers of the subcubes by $\{\cbb_k\}_{k=1}^{N^{d_1}}$ with $\cbb_k = [c_{k, 1}, c_{k, 2}, ..., c_{k, d_1}]^{\intercal}$.
Let $\{\cbb_k\}_{k=1}^{N^{d_1}}$ be a uniform grid on $Q_r$ (covering $\Omega_r$), so that $\cbb_k\in \left\{-r,-r+\frac{2r}{N-1},..., r\right\}^{d_1}$ for each $k$.
Define
\begin{align}
    \psi(a) = \begin{cases}
        1, |a|<1,\\
        0, |a|>2, \\
        2-|a|, 1\leq |a|\leq 2,
    \end{cases}
    \label{eqn_psi}
\end{align}
with $a\in\mathbb{R}$, and
\begin{align}
    \phi_{\cbb_k}(\xb) = \prod_{j = 1}^{d_1} \psi \left(\frac{3(N-1)}{2r}(x_j-c_{k,j})\right),\quad \xb\in Q_r.
    \label{eqn_phi}
\end{align}
In this definition, we have $\supp(\phi_{\cbb_k})=\left\{\xb: \|\xb-\cbb_k\|_{\infty}\leq \frac{4r}{3(N-1)}\right\}\subset \left\{\xb: \|\xb-\cbb_k\|_{\infty}\leq \frac{2r}{(N-1)}\right\}$ and \red{on the constrained domain} $\Omega_r$, we have
$$
\|\phi_{\cbb_k}\|_{L^{\infty}(Q_r)}=1, \quad \sum_{k=1}^{N^{d_1}} \phi_{\cbb_k}=1.
$$
For any $g(X_t)$ with $X_t\in\Omega_r$, we construct a piecewise constant approximation as below,
$$
\bar{g}(X_t)=\sum_{k=1}^{N^{d_1}} g(\cbb_k)\phi_{\cbb_k}(X_t), \quad X_t\in\Omega_r.
$$
Based on the decomposition of the domain $\mathbb R^{d_1}=\Omega_r\cup \Omega_r^C$, for any $T\ge 0$, we have
\begin{align*}
   & \mathbb E\Big[\sup_{0\le t\le T} \Big|g(X_t)-\bar{g}(X_t) \Big|^2 \Big]\\
   =&  \underbrace{ \mathbb E\Big[\sup_{0\le t\le T} \Big|g(X_t) \Big|^2 \mathbf{1}_{\Omega_r^C}(X_t)\Big]}_{\mathcal I_1}+\underbrace{\mathbb E\Big[\sup_{0\le t\le T} \Big|g(X_t)-\bar{g}(X_t) \Big|^2\mathbf{1}_{\Omega_r}(X_t) \Big]}_{\mathcal I_2}.
\end{align*}
For the first term $\mathcal I_1$, applying the polynomial growth Assumption \ref{assumption: polynomial} for \red{the} function $g$\red{, the} uniform bound Assumption \ref{assumption_X: power}, and \red{the Cauchy--Schwarz} inequality, we have
\begin{align}
 \mathcal{I}_1&=\mathbb E\Big[\sup_{0\le t\le T} \Big|g(X_t) \Big|^2 \mathbf{1}_{\Omega_r^C}(X_t)\Big]\nonumber\\
 &\le \mathbb E\Big[\sup_{0\le t\le T} \Big|g(X_t) \Big|^4 \Big]^{1/2} \mathbb E\Big[ \sup_{0\le t\le T}  (\mathbf{1}_{\Omega_r^C}(X_t))^{2}\Big]^{1/2}\nonumber \\
 &\le  \mathbb E\Big[\sup_{0\le t\le T} \Big|1+\red{|X_t|^p} \Big|^4 \Big]^{1/2} \Big(\mathbb P\Big(\sup_{0\le t\le T}|X_t|\ge r
 \Big) \Big)^{1/2}\nonumber\\
 &\le  C_0 \Big(\mathbb P\Big(\sup_{0\le t\le T}|X_t|\ge r
 \Big) \Big)^{1/2}.\nonumber
 \end{align}
 According to the tail bound Assumption \ref{assumption_X: tail}, by taking the initial point $X_0$ as the center of the domain, and selecting
\textcolor{black}{ $r=\left\lceil -\frac{2C_T}{c}\log \frac{\varepsilon}{4C_0}\right\rceil^{\frac{1}{\alpha}}$},
 we get the following bound
 \begin{align}
 \mathcal I_1&\le C_0e^{-\frac{cr^{\alpha}}{2C_T}}\le \frac{\varepsilon}{4},\label{est: I1}
\end{align}
\textcolor{black}{ where $C_0$ denotes \red{an} upper bound of $\mathbb E\Big[\sup_{0\le t\le T} \Big|1+\red{|X_t|^p} \Big|^4 \Big]^{1/2} $, which follows from Assumption \ref{assumption_X: power}. }
 For the second term $\mathcal{I}_2$, by applying the partition of unity property  $\sum_{k=1}^{N^{d_1}} \phi_{\cbb_k}=1$, we have
\begin{align}
  \mathcal I_2 =& \mathbb E\Big[\sup_{0\le t\le T} \Big|\sum_{k=1}^{N^{d_1}} [g(X_t)-g(\cbb_k)]\phi_{\cbb_k}(X_t) \Big|^2 \mathbf{1}_{\Omega_r}(X_t) \Big]\nonumber\\
  \leq &\mathbb E\Big[\sup_{0\le t\le T} \Big(\sum_{k=1}^{N^{d_1}} |g(X_t)-g(\cbb_k)||\phi_{\cbb_k}(X_t)|\Big)^2 \mathbf{1}_{\Omega_r}(X_t)\Big] \nonumber \\
  = & \mathbb E\Big[\sup_{0\le t\le T} \Big( \sum_{k: \|\cbb_k- X_t\|_{\infty} \leq \frac{2r}{(N-1)}}  |g(X_t)-g(\cbb_k)|\phi_{\cbb_k}(X_t)\Big)^2 \mathbf{1}_{\Omega_r}(X_t)\Big] , \nonumber\\
  \leq & \mathbb E\Big[\sup_{0\le t\le T} \Big( \max_{k: \|\cbb_k-X_t\|_{\infty} \leq \frac{2r}{(N-1)}}  |g(X_t)-g(\cbb_k)|\Big(\sum_{k: \|\cbb_k-X_t\|_{\infty} \leq \frac{2r}{(N-1)}}  \phi_{\cbb_k}(X_t) \Big)\Big)^2\mathbf{1}_{\Omega_r}(X_t)\Big] \nonumber \\
   \leq & \mathbb E\Big[\sup_{0\le t\le T} \max_{k: \|\cbb_k-X_t\|_{\infty} \leq \frac{2r}{(N-1)}}  |g(X_t)-g(\cbb_k)|^2\mathbf{1}_{\Omega_r}(X_t)\Big] \nonumber \\
   \leq & \frac{(2rL_g)^2 d_1 }{(N-1)^2},\label{est: I2}
\end{align}
where we use the Lipschitz assumption \red{on} $g$ in the last inequality and the uniform bound given $\{X_t\}_{0\le t\le T} \in\Omega_r$.
\red{Letting}  \textcolor{black}{$N=\left\lceil \frac{4\sqrt{d_1}\log^{1/\alpha}(\varepsilon^{-1}) L_g}{\sqrt{\varepsilon}} \right\rceil+1$}, we get
\begin{align}\label{est: I2-2}
\mathcal{I}_2\le \varepsilon/4.
\end{align}
Combining the estimates \eqref{est: I1} and \eqref{est: I2}, we get
\begin{align}\label{funtion 1/2 var}
 \mathbb E\Big[\sup_{0\le t\le T} \Big|g(X_t)-\bar{g}(X_t) \Big|^2 \Big]\le \frac{\varepsilon}{2}.
\end{align}

We then show that $\phi_{\cbb_k}$ can be approximated by a network with arbitrary accuracy on the compact domain $\Omega_r$ with a fixed parameter $r$. Notice that for a compact domain $\Omega_r$, this type of approximation has already been established in \cite{liu2024neural}. Since a different norm is employed here, we provide the proof for completeness.
Note that $\phi_{\cbb_k}$ is the product of $d_1$ functions, each of which is piecewise linear and can be realized by \red{a constant-depth ReLU network}.
Let $\widetilde{\times}$ be the network defined in Lemma \ref{lemma_pp3} with accuracy $\delta>0$.
For any $\xb\in Q_r$, we approximate $\phi_{\cbb_k}$ with $\widetilde{q}_k$ defined as follows,
 \begin{align*}
        \widetilde{q}_{k}(\xb)= \widetilde{\times}\left(\psi\left(\frac{3(N-1)}{2r}(x_1-c_{k,1})\right), \widetilde{\times}\left(\psi\left(\frac{3(N-1)}{2r}(x_2-c_{k,2})\right), \cdot\cdot\cdot \right)\right).
    \end{align*}

 For each $k$, define $\widetilde{q}_k\in \cF_{\rm NN}(d_1, 1, \cL, \mathfrak p, K, \kappa, R)$ with sizes to be specified later.
For any $X_t\in \Omega_r$\red{,} which is equivalent to considering $X_t\cdot\mathbf{1}_{\Omega_r}\red{(X_t)}$, we may simply take $\xb\in\Omega_r$. That is, by viewing $\xb$ as an arbitrary point in the domain $\Omega_r$, we have
    \begin{align*}
        &|\widetilde{q}_{k}(\xb)-\phi_{\cbb_{k}}(\xb) |\\
        \leq &\left|\widetilde{\times}\left(\psi\left(\frac{3(N-1)}{2r}(x_1-c_{k,1})\right), \widetilde{\times}\left(\psi\left(\frac{3(N-1)}{2r}(x_2-c_{k,2})\right), \cdot\cdot\cdot \right)\right) - \phi_{\cbb_{k}}(\xb) \right|\\
        \leq & \bigg|\widetilde{\times}\left(\psi\left(\frac{3(N-1)}{2r}(x_1-c_{k,1})\right), \widetilde{\times}\left(\psi\left(\frac{3(N-1)}{2r}(x_2-c_{k,2})\right), \cdot\cdot\cdot \right)\right) \\
        & - \psi\left(\frac{3(N-1)}{2r}(x_1-c_{k,1})\right)\widetilde{\times}\left(\psi\left(\frac{3(N-1)}{2r}(x_2-c_{k,2})\right), \cdot\cdot\cdot \right) \bigg|\\
        &+ \bigg| \psi\left(\frac{3(N-1)}{2r}(x_1-c_{k,1})\right)\widetilde{\times}\left(\psi\left(\frac{3(N-1)}{2r}(x_2-c_{k,2})\right), \cdot\cdot\cdot \right)  -\phi_{\cbb_{k}}(\xb) \bigg|\\
        \leq & \delta + \mathcal{E}_2,
    \end{align*}
    where
    \begin{align*}
        \mathcal{E}_2 &= \bigg|\psi\left(\frac{3(N-1)}{2r}(x_1-c_{k,1})\right)\widetilde{\times}\left(\psi\left(\frac{3(N-1)}{2r}(x_2-c_{k,2})\right), \cdot\cdot\cdot \right)  -\phi_{\cbb_{k}}(\xb) \bigg|\\
        & = \bigg| \psi\left(\frac{3(N-1)}{2r}(x_1-c_{k,1})\right) \bigg|
       \bigg| \widetilde{\times}\left(\psi\left(\frac{3(N-1)}{2r}(x_2-c_{k,2})\right), \cdot\cdot\cdot \right) - \prod_{j = 2}^{d_1} \psi \left(\frac{3(N-1)}{2r}(x_j-c_{k,j})\right) \bigg|\red{.}
    \end{align*}

\red{Repeating} this process to estimate $\mathcal{E}_2, \mathcal{E}_3, ..., \mathcal{E}_{d_1+1}$, where $\mathcal{E}_{d_1+1} = \red{\prod\limits_{j = 1}^{d_1}\psi\left(\frac{3(N-1)}{2r}(x_j-c_{k,j})\right) - \phi_{\cbb_k}(\xb)} = 0$\red{, we obtain} $\|\phi_{\cbb_{k}} - \widetilde{q}_{k}\|_{L^{\infty}(\Omega_r)}\leq d_1\delta$.
Thus, we have
\begin{align*}
 \E \left[\sup_{0\le t\le T} |\phi_{\cbb_{k}}(X_t) - \widetilde{q}_{k}(X_t)|^2\mathbf{1}_{\Omega_r}(X_t)\right]\le d^2_1 \delta^2.
\end{align*}

Applying \red{the Cauchy--Schwarz} inequality and Assumption \ref{assumption: polynomial}, we have

 \begin{align}
    & \mathbb E \left[ \sup_{0\le t\le T}  \left| \sum_{k=1}^{N^{d_1}} g(\cbb_k)\widetilde{q}_k(X_t) - \bar{g}(X_t) \right|^2 \mathbf{1}_{\Omega_r}(X_t)\right]\nonumber\\
        = &\mathbb{E}\left[\sup_{0\le t\le T} \left| \sum_{k=1}^{N^{d_1}} g(\cbb_k)\widetilde{q}_k(X_t) - \sum_{k=1}^{N^{d_1}} g(\cbb_k)\phi_{\cbb_k}(X_t) \right|^2\mathbf{1}_{\Omega_r}(X_t)\right] \nonumber \\
        \leq &  \mathbb{E}\left[\sum_{k=1}^{N^{d_1}}|g(\cbb_k)|^2 \sup_{0\le t\le T}\sum_{k=1}^{N^{d_1}} |\widetilde{q}_k(X_t)-\phi_{\cbb_k}(X_t)|^2 \mathbf{1}_{\Omega_r}(X_t)\right]\nonumber \\
        \leq &  d_1^2N^{2d_1}C_g^2(1+r^p)^2\delta^2\le \frac{\varepsilon}{2},
\label{eqn_estimate_of_delta_in_function}
    \end{align}
    where the \red{second-to-last} inequality follows from the fact \red{that} $X_t\in\Omega_r$ and the polynomial growth of $g$\red{, and the last inequality holds by selecting} \textcolor{black}{$\delta =\sqrt{ \frac{\varepsilon}{2d_1^2N^{2d_1}C_g^2(1+r^p)^2}} = O\red{\Big(}d_1^{-\frac{d_1}{2}} \varepsilon^{\frac{d_1+1}{2}}\left( \log(\varepsilon^{-1})\right)^{-\frac{d_1+p}{\alpha}}\red{\Big)}$}.
Thus, we have
\begin{align*}
& \mathbb E\left[\sup_{0\le t\le T}   \left| g(X_t)-\sum_{k=1}^{N^{d_1}} g(\cbb_k)\widetilde{q}_k\red{(X_t)}\right|^2\right] \\
\leq & \mathbb E\Big[\sup_{0\le t\le T}  \left| g(X_t)-\bar{g}(X_t)\right|^2\Big] + \mathbb E\left[\sup_{0\le t\le T}  \left| \bar{g}(X_t)-\sum_{k=1}^{N^{d_1}} \red{g}(\cbb_k)\widetilde{q}_k\red{(X_t)}\right|^2 \right] \\
\leq &\frac{\varepsilon}{2}+ \frac{\varepsilon}{2}=\varepsilon.
\end{align*}

The network architecture is then specified in the theorem.
\end{proof}

\subsection{Functional Approximations}

\begin{theorem}\label{thm_functional}
Let $\sF$ be \red{as} in Definition \ref{def.lip_functional}, and assume Assumptions \ref{assumption_X: power}, \ref{assumption_X: tail}, \ref{assumption: polynomial}, and \ref{assumption_input} hold.
     For any $\varepsilon>0$,
     set \textcolor{black}{$  r =\left\lceil -\frac{2C_T}{c}\log \frac{\varepsilon^2}{8C_0L_\sF^2}\right\rceil^{\frac{1}{\alpha}}+1$}.
   Let $\{\cbb_m\}_{m=1}^{N^\delta}\subset Q_r$ so that $\{\mathcal{B}_{\delta}(\cbb_m) \}_{ m  = 1}^{N^{\delta}}$ is a cover of $Q_r$ for some $N^\delta>0$ to be estimated in Remark \ref{remark_functional_estimate_N_H},
   and \textcolor{black}{with ball radius $\delta = C\varepsilon (L_\sF L_g)^{-1}d_1^{-\frac{1}{2}} $}, \red{where $C$ is} a constant.
     Let $H=O(\sqrt{N^\delta}^{N^{\delta}}\varepsilon^{-N^{\delta}}) $, and set the network $\cF_{\rm NN}(\red{N^\delta}, 1 , \red{\mathcal L}, \mathfrak p, K, \kappa, R)$
     with
        \begin{align*}
        &\mathcal{L} = O\left((N^\delta)^2 \log(N^{\delta}) +(N^\delta)^2\log(\varepsilon^{-1})\right), \mathfrak{p} = O(1), K = O\left((N^\delta)^2\log N^\delta+(N^\delta)^2\log(\varepsilon^{-1})\right), \\ &\kappa=O((N^\delta)^{-\frac{N^\delta}{2}}\varepsilon^{\frac{-N^\delta-1}{2}}), R = 1.
    \end{align*}
    There are
     $\{\widetilde{q}_k\}_{k=1}^{H}$ with $ \widetilde{q}_k \in \cF_{\rm NN}(N^\delta, 1, \mathcal{L}, \mathfrak{p}, K, \kappa, R)$ for any $k$, such that we have
    \begin{align}
         \sup_{g \in \cG} |\sF g - \sum_{k=1}^H a_k \tilde{q}_k (\bg)| \leq \varepsilon,
    \end{align}

    where $\bg=[g(\cbb_1), g(\cbb_2),...,g(\cbb_{N^\delta})]^{\top}$, \red{and the} $a_k$'s are coefficients depending on $\sF$.
    The constant hidden in $O$ and all constants $C$ depend on the constants $L_{\sF}, L_g, C_T, C_0, c, \alpha$ in the assumptions.

\end{theorem}

\begin{proof}

For $r>0$, define the cube as before,
\begin{equation*}
Q_r := [-r,r]^{d_1} \;=\; \left\{ x\in\mathbb{R}^{d_1} : \|x\|_{\infty} \red{\leq} r \right\},
\qquad \mbox{where }
\|x\|_{\infty} := \max_{1\le j\le \red{d_1}} |x_j| .
\end{equation*}
Here $r$ is chosen so that \red{the} tail probability is small\red{, as} we did in the function approximation. Let $\{B_\delta(c_k)\}_{k=1}^{N^\delta}$  be a finite cover of $Q_r$ by $N^\delta$ Euclidean balls.
By Lemma \ref{lemma_pou}, there exists a partition of unity $\{w_k(x)\}_{k=1}^{\red{N^\delta}}$ subordinate to the cover $\{B_{\delta}(c_k)\}_{k=1}^{N^\delta}$.

For any $g\in\red{\mathcal{G}}$, define $\bg=[g(c_1),g(c_2),\ldots,g(c_{N^{\delta}})]^{\mathsf T}$, and define,
\begin{align}
\label{eq:gw-def}
g_w(x)=
\begin{cases}
\displaystyle \sum_{k=1}^{N^\delta} g(c_k)\,w_k(x), & x\in \Omega_r,\\[0.8ex]
0, & x\in \Omega_r^{\,C}\, .
\end{cases}
\end{align}

Then, \red{similarly} to the estimates in \eqref{funtion 1/2 var}, \red{since} $N^\delta$ is the covering number \red{in} all $d_1$ dimensions, we have
\begin{align}
&\mathbb{E}\!\left[ \sup_{t\le s\le T} \bigl| g(X_s) - g_w(X_s) \bigr|^{2} \right]\nonumber
\\
=&   \underbrace{ \mathbb E\Big[\sup_{t\le s\le T} \Big|g(X_s) \Big|^2 \mathbf{1}_{\Omega_r^C}(X_s)\Big]}_{\mathcal J_1}+\underbrace{\mathbb E\Big[\sup_{t\le s\le T} \Big|g(X_s)-g_w(X_s) \Big|^2\mathbf{1}_{\Omega_r}(\red{X_s}) \Big]}_{\mathcal J_2}.\label{est J1J2}
\end{align}
\red{Similarly} to \eqref{est: I1},  we have
\begin{align}
    \label{est: J1}
    \mathcal J_1\le C_0e^{-\frac{cr^{\alpha}}{2C_T}}.
\end{align}
\red{Similarly} to \eqref{est: I2}, taking into \red{account that} the radius of the covering \red{balls is} $\delta$, we have
\begin{align}
  \mathcal J_2 =& \mathbb E\Big[\sup_{t\le s\le T} \Big|\sum_{k=1}^{N^\delta} [g(X_s)-g(\cbb_k)]w_k(X_s) \Big|^2 \mathbf{1}_{\Omega_r}(X_s) \Big]\nonumber\\
  \leq &\mathbb E\Big[\sup_{t\le s\le T} \Big(\sum_{k=1}^{N^\delta} |g(X_s)-g(\cbb_k)||w_k(X_s)|\Big)^2 \mathbf{1}_{\Omega_r}(X_s)\Big] \nonumber \\
  = & \mathbb E\Big[\sup_{t\le s\le T} \Big( \sum_{k: \|\cbb_k- X_s\|_{\infty} \leq \delta}  |g(X_s)-g(\cbb_k)|w_k(X_s)\Big)^2 \mathbf{1}_{\Omega_r}(X_s)\Big] , \nonumber\\
  \leq & \mathbb E\Big[\sup_{t\le s\le T} \Big( \max_{k: \|\cbb_k-X_s\|_\infty \leq \delta}  |g(X_s)-g(\cbb_k)|\Big(\sum_{k: |\cbb_k-X_s| \leq \delta}  w_k(X_s) \Big)\Big)^2\mathbf{1}_{\Omega_r}(X_s)\Big] \nonumber \\
   \leq & \mathbb E\Big[\sup_{t\le s\le T} \max_{k: \|\cbb_k-X_s\|_{\infty} \leq \delta}  |g(X_s)-g(\cbb_k)|^2\mathbf{1}_{\Omega_r}(X_s)\Big] \nonumber \\
   \leq & (L_g\delta)^2 d_1.\label{est: J2}
\end{align}
From the Lipschitz property of the functional $\sF$ defined in Definition \ref{def.lip_functional}, we have
\begin{align}
\bigl|\sF(g)- \sF(g_w)\bigr|^{2}
&\le L_\sF^2\,\|g-g_w\|_{S^2}^{2}\nonumber \\
&\le\; L_\sF^{2}\left((L_g\delta)^2 d_1+\; C_0e^{-\frac{cr^{\alpha}}{2C_T}} \right)\  < \varepsilon^2 / 4, \label{est: F 1}
\end{align}
i.e.
\begin{align}
    \bigl|\sF(g)- \sF(g_w)\bigr|<\varepsilon/2,
\end{align}
which follows by selecting
\begin{align}
\label{eqn_functional_r_delta}
    r =\left\lceil -\frac{2C_T}{c}\log \frac{\varepsilon^2}{8C_0L_\sF^2}\right\rceil^{\frac{1}{\alpha}}+1, \quad \delta = \frac{\varepsilon }{2\sqrt{2}L_\sF L_gd_1^{\frac{1}{2}} }.
\end{align}
Now, for any $g,\tilde g\in\mathcal{G}$, define $g_w$ and $\tilde g_w$ as in \eqref{eq:gw-def}, and set
\begin{equation*}
\bg=\bigl[g(c_1),\ldots,g(c_{N^{\delta}})\bigr]^{\mathsf T},
\qquad
\tilde \bg=\bigl[\tilde g(c_1),\ldots,\tilde g(c_{N^{\delta}})\bigr]^{\mathsf T}.
\end{equation*}
Define the function $h(\bg):= \sF(g_w)$.
Then
\begin{align}
|h(\bg)-h(\tilde \bg)|^{2}
&= \bigl| \sF(g_w)- \sF(\tilde g_w) \bigr|^{2}\nonumber\\
& \le L_\sF^2\,\|g_w-\tilde g_w\|_{S^2}^{2}\nonumber\\
&= L_\sF^2\,\mathbb{E}\!\left[ \sup_{t\le s\le T} \!\!
\Bigl( \sum_{k=1}^{N^\delta} \bigl(g(c_k)-\tilde g(c_k)\bigr)\, w_k(X_s) \Bigr)^{2} \right]\nonumber\\
&\le L_\sF^2  \,\Bigl( \sum_{k=1}^{N^\delta} |g(c_k)-\tilde g(c_k)|^{2} \Bigr)\,
\mathbb{E}\!\left[ \sup_{t\le s\le T} \sum_{k=1}^{N^\delta} w_k(X_s)^{2} \right]\nonumber \\
&\le L_\sF^2  \,\,\Bigl( \sum_{k=1}^{N^\delta} |g(c_k)-\tilde g(c_k)|^{2} \Bigr)\nonumber\\
&\le L_\sF^2  \,\,|\bg - \tilde{\bg}|^2 ,\label{lip est h}
\end{align}
where we use the fact that $\{w_k(x)\}_{k=1}^{N^\delta}$ is a partition of unity. Thus we \red{have shown} that $h(\bg):= \sF(g_w)$ is a \red{Lipschitz} function on $\mathcal{G}$ according to \eqref{lip est h}.\red{ }Besides, according to  the assumption on $\sF$ and the definition of $h$, $h$ is bounded.
Also, the domain of $h$ is bounded by the range of $g$.
Consequently, $h$ satisfies the \red{conditions of the} approximation rate estimate in \red{\cite[Theorem 5]{liu2024neural}. It} follows that, for $\varepsilon>0$, if we set \textcolor{black}{$H=C\sqrt{N^{\delta}}^{N^{\delta}}\varepsilon^{-N^\delta}$} for some $C$,  then there exists a
    network architecture $\cF_{\rm NN}(N^{\delta},1, \mathcal{L}, \mathfrak{p}, K,\kappa, R)$ and $\{\widetilde{q}_k\}_{k=1}^{ H}$ with $\widetilde{q}_k\in \cF_{\rm NN}(N^{\delta},1, \mathcal{L}, \mathfrak{p},K,\kappa, R)$ for $k=1,\ldots,H$ such that
    \begin{align}
        \sup_{g \in \cG}\left|\sF(g_w)-\sum_{k=1}^{H} a_k \widetilde{q}_k(\bg )\right|\leq \frac{\varepsilon}{2},\label{functional}
    \end{align}
    where \red{the} $a_k$ are constants depending on $\red{\sF}$.
    \textcolor{black}{Such an architecture has}
    \begin{align*}
        &\mathcal{L} = O\left((N^\delta)^2 \log(N^{\delta}) +(N^\delta)^2\log(\varepsilon^{-1})\right), \mathfrak{p} = O(1), K = O\left((N^\delta)^2\log N^\delta+(N^\delta)^2\log(\varepsilon^{-1})\right), \\ &\kappa=O((N^\delta)^{-\frac{N^\delta}{2}}\varepsilon^{\frac{-N^\delta-1}{2}}), R = 1 .
    \end{align*}
Finally, we have, for any $g  \in \mathcal{G}$ and $\bg =[g(c_1),...,g(c_{N^\delta})]^{\top}$\red{,}
    \begin{align*}
        \sup_{g \in \cG} \left|\sF(g)-\sum_{k=1}^{H} a_k \widetilde{q}_k(\bg )\right| \leq
        &\sup_{g\in \cG} \left|\sF(g)-h(\bg) \right| + \sup_{\bg}\left| h(\bg)-\sum_{k=1}^{H} a_k \widetilde{q}_k(\bg )\right|\\
        \leq &\frac{\varepsilon}{2}+ \frac{\varepsilon}{2}=\varepsilon.
    \end{align*}
\end{proof}
The next lemma and remark are used to estimate $N^{\delta}$ and $H$.
\begin{lemma}\label{lem.cover.ball}
    Let $\mathfrak D=[-\gamma,\gamma]^d$ for some $\gamma>0$. For any $\delta>0$, there exists a cover $\{\mathcal{B}_{\delta}(\cbb_m) \}_{ m  = 1}^{M}$ of $\mathfrak D$ with
    \begin{align}
        \textcolor{black}{M\leq C\delta^{-d}},
        \label{eqn_chapter2}
    \end{align}
    where $C$ is a constant depending on $\gamma$ and $d$.
\end{lemma}
\begin{proof}[Proof of Lemma \ref{lem.cover.ball}]
     By \cite[Chapter 2]{conway2013sphere}, we have
\begin{align}
\textcolor{black}{\red{M}\leq \left\lceil\frac{2\gamma}{\delta}\right\rceil^{d}+7d\log d\leq C\left(\frac{\gamma}{\delta}\right)^{d}}
        \label{eqn_chapter2_proof}
    \end{align}
    for some $C$ depending on $\gamma$ and $d$.
 \end{proof}

\begin{remark}
\label{remark_functional_estimate_N_H}
    In this remark, we estimate the \red{covering number} $N^\delta$ and hence the number \red{$H$ of basis functions} needed.
    By Lemma \ref{lem.cover.ball} and \red{\eqref{eqn_functional_r_delta}} for the \red{estimates of $r$ and $\delta$} in the proof, it follows that
    \begin{align}
    \label{eqn_functional_N_estimtae}
        N^\delta\leq C\left(\frac{r }{\delta}\right)^{d_1}
        &\leq
        C(-\frac{2C_{T}}{c})^{\frac{d_1}{\alpha}}\left(\log\frac{\varepsilon^2}{8C_0L_\sF^2} \right)^{\frac{d_1}{\alpha}}\varepsilon^{-d_1} d_1^{\frac{d_1}{2}} \nonumber\\
        &\leq C \left(\log (\varepsilon^{-2}) \right)^{\frac{d_1}{\alpha}}\varepsilon^{-d_1},
    \end{align}
    where the constant $C$ depends on $C_T$, $c$, $C_0$, $\alpha$, $L_\sF, L_g$ and $d_1$.
    Dropping the \red{lower-order} term in \red{\eqref{eqn_functional_N_estimtae}}, it follows that
    \begin{align*}
        H = C(N^{\delta})^{\frac{1}{2}N^{\delta} }\varepsilon^{-N^\delta} = C\varepsilon^{-\frac{d_1}{2} \varepsilon^{-d_1}}\varepsilon^{-\varepsilon^{-d_1}},
    \end{align*}
    or \textcolor{black}{$H = \mathcal{O}(\varepsilon^{-(\frac{d_1}{2}+1)\varepsilon^{-d_1  }})$}.
\end{remark}

\subsection{Operator Approximation}
\begin{theorem}[Operator]\label{thm_operator}
Let Assumptions \ref{assumption_X: power}, \ref{assumption_X: tail}, \ref{assumption: polynomial}, \ref{assumption_input}, \ref{assumption_U}, and \ref{assumption_G} hold.
    For any $\varepsilon>0$,
    set $N_1 = C\varepsilon^{-d_2}$, and $N_2 = O(\varepsilon^{-d_1d_2-d_1 })$.
    Define the network architecture $\cF_1=\cF_{\rm NN}(d_2,1,\cL_1,\mathfrak p_1,K_1,\kappa_1)$ and $\cF_2=\cF_{\rm NN}(N_2,1,\cL_2,\mathfrak p_2,K_2,\kappa_2)$ with
\begin{align*}
       \mathcal L_1 = O\left(\log(\varepsilon^{-2})\right),\ \mathfrak p_1 = O(1),\ K_1 = O\left(\log(\varepsilon^{-2})\right), \ \kappa_1=O(\varepsilon^{-d_2}), R = 1,
    \end{align*}
and
 \begin{align*}
        &\mathcal L_2 = O\left(N_2^2\log N_2 + N_2^2\log(\varepsilon^{-d_2 - 1})\right),\ \mathfrak p_2 = O(N_2^{\frac{1}{2}}\varepsilon^{-N_2d_2 - N_2}),\\
        & K_2 = O\left(N_2^{\frac{1}{2}}\varepsilon^{-N_2d_2 - N_2}\left( N_2^2\log N_2+N_2^2\log(\varepsilon^{-d_2-1}) \right)\right), \\
        &\kappa_2 = O(N_2^{-\frac{N_2}{2}}\varepsilon^{-\frac{N_2d_2}{2} - N_2 }).
    \end{align*}
    \color{black}
For any operator $\Gamma:\mathcal G\rightarrow \mathcal U$ satisfying Assumption \ref{assumption_G}, there are $\{\widetilde{q}_k\}_{k=1}^{N_1}$ with $\widetilde{q}_k  \in \mathcal{F}_1$ and $\{\widetilde{a}_k\}_{k=1}^{N_1} $ with $\widetilde{a}_k \in \mathcal{F}_2$ such that
\begin{align}
        \sup_{g\in \mathcal G} \E\bigg[ \sup_{0 \leq t \leq T} \left|\Gamma (g)(X_t)-\sum_{k=1}^{N_1} \widetilde{a}_k(\bg )\widetilde{q}_k(X_t)\right| \bigg] \leq \varepsilon.
    \end{align}
\end{theorem}

\begin{proof}[Proof of Theorem \ref{thm_operator}]
By \red{Assumptions} \ref{assumption_U} and \ref{assumption_G}, $\Gamma(g)(\cdot)$ satisfies the assumptions of Theorem \ref{thm_function}.
It follows that, for $\varepsilon_1>0$ \red{to be specified later, setting} $r=\left\lceil -\frac{2C_T}{c}\log \frac{\varepsilon_1}{4C_0}\right\rceil^{\frac{1}{\alpha}}$ and \red{defining} $Q_r, \Omega_r$ as in Definition \ref{def: omega r},
    there \red{exist} a constant $N_1 = C\varepsilon_1^{-d_2}$ for some constant $C$ depending on $d_2, L_g$ and $r$, a network architecture $\mathcal{F}_1=\cF_{\rm NN}(d_2, 1 ,\cL_1, \mathfrak p_1, K_1, \kappa_1, R_1)$\red{, networks}  $\{\widetilde{q}_k\}_{k = 1}^{N_1}$ with $\widetilde{q}_k\in \mathcal{F}_1$, and \red{points} $\{\cbb_k\}_{k = 1}^{N_1} \subset Q_r$ such that for any $g\in \mathcal G$, we have
    \begin{align}
        \E \bigg[ \sup_{0 \leq t \leq T} \left|\Gamma (g)(X_t)-\sum_{k = 1}^{N_1} \Gamma (g)(\cbb_k) \widetilde{q}_k(X_t)\right|^2 \bigg]^{1/2}\leq \varepsilon_1.
        \label{eq.operator.1}
    \end{align}
    Such a network has size
\begin{align*}
\mathcal L_1 = O\left(\log(\varepsilon_1^{-2})\right),\ \mathfrak p_1 = O(1),\ K_1 = O\left(\log(\varepsilon_1^{-2})\right), \ \kappa_1=O(\varepsilon_1^{-d_2}).
    \end{align*}
For each $k$, define the functional
\begin{align}
    \sF_{k}(\Gamma g) :=\Gamma g(\cbb_k).
\label{eq:fk}
\end{align}
For any $g_1,g_2\in \cG$, \red{when} the forward process at \red{time} $t$ starts at $\cbb_k$, we have
\begin{align}
    & |\sF_{k}(\Gamma g_1)- \sF_{k}(\Gamma  g_2)|^2 \nonumber\\
    \leq & \mathbb E \bigg[ \sup_{0 \leq s \leq T}|\Gamma (g_1)(X_s)- \Gamma  (g_2)(X_s)|^2 \bigg] \nonumber\\
    \leq & L_\Gamma^2  \mathbb E \bigg[ \sup_{0 \leq s \leq T} |g_1(X_s) - g_2(X_s)|^2 \bigg] \label{eq.proof.approx.fLip}\\
    = & L_\Gamma^2 \|g_1 - g_2\|_{S^2}^2,\nonumber
\end{align}
where the last inequality follows from Assumption \ref{assumption_G}.
By Theorem \ref{thm_functional}, for any $\varepsilon_2>0$, there exist $N_2$ and $H$ with values estimated later, and a network architecture $\cF_2=\cF_{\rm NN}(N_2, 1 , \cL_2, \mathfrak p_2, K_2, \kappa_2, R_2)$ with
     \begin{align*}
        &\mathcal L_2 = O\left(N_2^2\log N_2+N_2^2\log(\varepsilon_2^{-1})\right),\ \mathfrak p_2 = O(1), \ K_2 = O\left(N_2^2\log N_2+N_2^2\log(\varepsilon_2^{-1})\right), \\
        &\kappa_2=O(N_2^{-\frac{N_2}{2}} \varepsilon_2^{-N_2-1}), R = 1.
    \end{align*}
Such a network architecture gives a network $\widetilde{a}_k$ so that
\begin{align*}
    \sup_g |\sF_{k}(\Gamma (g))-\widetilde{a}_{k}(\bg)|\leq \varepsilon_2.
\end{align*}
Therefore,
\begin{align}
        &\E \bigg[\sup_{0 \leq s \leq T} \left|\sum_{k = 1}^{N_1} \sF_{k}(\Gamma g)\widetilde{q}_k(X_s) - \sum_{k = 1}^{N_1} \widetilde{a}_k(\bg ) \widetilde{q}_k(X_s)\right| \bigg] \nonumber\\
        =& \E \bigg[\sup_{0 \leq s \leq T} \left|\sum_{k = 1}^{N_1} \left(\sF_k(\Gamma g)-\widetilde{a}_k(\bg) \right) \widetilde{q}_k(X_s)\right| \bigg]  \nonumber\\
        \leq& \sum_{k = 1}^{N_1} \sup_{\bg} |\sF_k(\Gamma  g) -\widetilde{a}_k(\bg )| = N_1\varepsilon_2.
        \label{eq.operator.2}
    \end{align}

Applying the Cauchy–Schwarz inequality, using (\ref{eq.operator.1}) and (\ref{eq.operator.2}), we have
    \begin{align*}
        &\sup_{g\in \cG} \E \bigg[ \sup_{0 \leq s \leq T} \left|\Gamma g(X_s)-\sum_{k = 1}^{N_1} \widetilde{a}_k(\bg )\widetilde{q}_k(X_s)\right| \bigg] \nonumber \\
        \leq & \sup_{g\in \cG} \E \bigg[ \sup_{0 \leq s \leq T} \left|\Gamma (g)(X_s)-\sum_{k = 1}^{N_1} \sF_k(\Gamma g) \widetilde{q}_k(X_s)\right| \bigg]\\
&        + \sup_{g\in \cG} \E \bigg[\sup_{0 \leq s \leq T} \left|\sum_{k = 1}^{N_1} \sF_{k}(\Gamma g)\widetilde{q}_k(X_s) - \sum_{k = 1}^{N_1} \widetilde{a}_k(\bg ) \widetilde{q}_k(X_s)\right| \bigg]\\
        \leq &\sup_{g\in \cG} \E \bigg[ \sup_{0 \leq s \leq T} \left|\Gamma (g)(X_s)-\sum_{k = 1}^{N_1} \sF_k(\Gamma g) \widetilde{q}_k(X_s)\right|^2 \bigg]^{1/2}\\
&        + \sup_{g\in \cG} \E \bigg[\sup_{0 \leq s \leq T} \left|\sum_{k = 1}^{N_1} \sF_{k}(\Gamma g)\widetilde{q}_k(X_s) - \sum_{k = 1}^{\red{N_1}} \widetilde{a}_k(\bg ) \widetilde{q}_k(X_s)\right| \bigg]\\
        \leq& \varepsilon_1 + N_1\varepsilon_2.
    \end{align*}
    \red{Setting} $\varepsilon_2 = \varepsilon_1/(2N_1)$ \red{and} $\varepsilon_1=\frac{\varepsilon}{2}$,
    it follows that $\varepsilon_2 = O(\varepsilon^{d_2+1})$, \red{and} we then have
    \begin{align*}
        &\sup_{g\in \cG} \E \bigg[ \sup_{0 \leq s \leq T} \left|\Gamma g(X_s)-\sum_{k = 1}^{N_1} \widetilde{a}_k(\bg )\widetilde{q}_k(X_s)\right| \bigg] \leq \varepsilon.
    \end{align*}
\textcolor{black}{By Remark \ref{remark_functional_estimate_N_H}, the resulting network architectures have $N_2 = O(\varepsilon^{-d_1d_2-d_1} )$,}
\red{and} the number \red{$H$ of basis functions} needed is then estimated as
\textcolor{black}{$H = O(N_2^{\frac{1}{2}N_2}\varepsilon_2^{-N_2})$}, which determines the width of $\mathcal{F}_2$, \red{namely} $\mathfrak p_2 = H$.
Consequently, the network size estimate follows\red{:}
 \begin{align*}
        &\mathcal L_2 = O\left(N_2^2\log N_2 + N_2^2\log(\varepsilon^{-d_2 - 1})\right),\ \mathfrak p_2 = O(N_2^{\frac{1}{2}N_2 }\varepsilon^{-N_2d_2 - N_2}),\\
        & K_2 = O\left(N_2^{\frac{1}{2}}\varepsilon^{-N_2d_2 - N_2}\left( N_2^2\log N_2+N_2^2\log(\varepsilon^{-d_2-1}) \right)\right), \\
        &\kappa_2 = O(N_2^{-\frac{N_2}{2}}\varepsilon^{-\frac{N_2d_2}{2} - N_2 }).
    \end{align*}
\end{proof}

\section{European Option Pricing Operator}\label{sec: pricing european}

After proving the universal approximation of the operator, we consider the following applications \red{to European- and American-type} option pricing problems in this section and the next section. The relationship \red{between} the functional, \red{the} operator and the solution of the BSDE is \red{described in the following table.}

\begin{center}
\renewcommand{\arraystretch}{1.4}
\begin{tabular}{|c|c|c|}
\hline
\textbf{Symbol} & \textbf{Meaning} & \textbf{Definition / Norm} \\
\hline
\(\cG \) & Input space (payoffs) &
\(\displaystyle
\mathbb E\!\left[\sup_{0 \le s \le T} |g(X_s)|^2\right] < \infty
\) \\[6pt]
\(\mathcal U\) & Output space (pricing functions) &
\(\displaystyle
\mathbb E\!\left[\sup_{0 \le s \le T} |u(s,X_s)|^2\right] < \infty
\) \\[6pt]
\(\Gamma\) & Pricing operator &
\(\Gamma: \cG \to \mathcal U, \quad g \mapsto u=\Gamma(g)\) \\[4pt]
\(\sF_{t,x}\) & Evaluation functional &
\(\sF_{t,x}(u) = u(t,x)\) \\[4pt]
\(Y_t\) & BSDE solution &
\(Y_t = (\sF_{t,X_t} \circ \Gamma)(g)\) \\[4pt]
\hline

\end{tabular}
\label{table}
\end{center}

\subsection{European option pricing}

Let $(\Omega, \mathcal{F}, \{\mathcal{F}_t\}_{t \in [0,T]}, \mathbb{Q})$ be a filtered probability space satisfying the usual conditions, carrying a $d$-dimensional Brownian motion $B=(B^1,\dots,B^d)$ under the risk-neutral measure $\mathbb{Q}$.

The state process $X_t \in \mathbb{R}^{d_1}$ follows the diffusion
\begin{equation}
dX_t = b(t, X_t)\,dt + \sigma(t, X_t)\,dB_t, \quad X_0 = x,
\end{equation}
where $b: [0,T] \times \mathbb{R}^{d_1} \to \mathbb{R}^{d_1}$ and $\sigma: [0,T] \times \mathbb{R}^{d_1} \to \mathbb{R}^{d_1 \times d}$ are measurable, locally bounded, and Lipschitz in $x$.
Let $g: \mathbb{R}^{d_1} \to \mathbb{R}$ be the terminal payoff, such that $g(X_T) \in L^2(\mathbb{Q})$.
The price of the European option is
\begin{equation}
u(t,x) = \mathbb{E}^{\mathbb{Q}}\!\left[\exp\!\left(-\int_t^T r(s, X_s)\,ds\right) \, g(X_T) \, \big| \, X_t = x\right]\red{,}
\end{equation}
where the \red{risk-free} rate $r(t,X_t) \ge 0$.

Let $\mathcal{L}$ denote the infinitesimal generator of $X_t$:
\begin{equation}
(\mathcal{L}\phi)(t,x) = \sum_{i=1}^{\red{d_1}} b_i(t,x)\partial_{x_i}\phi(t,x)
+ \frac{1}{2} \sum_{i,j=1}^{\red{d_1}} a_{ij}(t,x)\partial_{x_i x_j}^2 \phi(t,x)\red{,}
\end{equation}
where $a(t,x) = \sigma(t,x)\sigma(t,x)^\top$.
Then $u(t,x)$ satisfies the  PDE
\begin{equation}
\partial_t u + \mathcal{L}u - r u = 0, \qquad u(T,x) = g(x).
\end{equation}
In addition, from a \red{probabilistic} point of view, the price process $(Y_t, Z_t)$ satisfies the backward stochastic differential equation (BSDE):
\begin{equation}
Y_t = g(X_T) - \int_t^T r(s,X_s)Y_s\,ds - \int_t^T Z_s\,dB_s,
\end{equation}
where $Y_t = u(t, X_t)$ and $Z_t = \sigma^\top(t, X_t)\nabla_x u(t, X_t)$.

As an example, in the case of \red{the} Black--Scholes \red{model}, for a single asset $X_t$ with
\[
dX_t = X_t(r\,dt + \sigma\,dB_t),
\]
and payoff $g(X_T)$, the PDE reduces to
\begin{equation}
\partial_t u + \frac{1}{2}\sigma^2 x^2 u_{xx} + \red{r x u_x} - r u = 0, \quad u(T,x) = g(x).
\end{equation}

\begin{theorem}[Lipschitz continuity of the European pricing operator in $S^2$]
Let $(\Omega,\mathcal F,\{\mathcal F_t\}_{t\in[0,T]},\mathbb Q)$ support a $d$-dimensional
Brownian motion $B$, and let the state process $X$ solve
\begin{align}
    \label{sde model}dX_t=b(t,X_t)\,dt+\sigma(t,X_t)\,\red{dB_t},\qquad X_0=x,
\end{align}
with $b,\sigma$ Lipschitz in $x$ and of linear growth. Let the \red{risk-free} rate
$r(t,X_t)$ be bounded and nonnegative, with $0\le r(t,x)\le \bar r$.
For any terminal payoff $g:\mathbb R^{d_1}\to\mathbb R$ with
$\mathbb E[|g(X_T)|^2]<\infty$, let $u= \Gamma^E g$ denote the (unique) solution to
\[
\partial_t u+\mathcal L u-r\,u=0,\qquad u(T,\cdot)=g(\cdot),
\]
where $\mathcal L$ is the generator of $X$.
Set $Y^g_t:=u(t,X_t)$ and define the $S^2$-norm
$\|Y\|_{S^2}:=\big(\mathbb E[\sup_{0\le t\le T}|Y_t|^2]\big)^{1/2}$.

Then for any two terminal payoffs $g_1,g_2$, the operator $\Gamma^E$ satisfies the following condition,
\[
\mathbb E\!\left[\sup_{0\le t\le T}\!\big|\Gamma^E(g_1)(t,X_t)- \Gamma^E(g_2)(t,X_t)\big|^2\right]
\;\le\; L\,
\mathbb E\!\left[\sup_{0\le t\le T}\!\big|\red{g_1(X_t)-g_2(X_t)}\big|^2\right],
\]
with Lipschitz constant \(L=4e^{2\bar r T}\). Hence the pricing operator \(\red{\Gamma^E}\) is Lipschitz on $S^2$.
\end{theorem}

\begin{proof}
By Feynman--Kac, for each $g$ we have (under $\mathbb Q$)
\[\Gamma^E(g):=
Y^g_t
= u(t,X_t)
= \mathbb E\!\left[\exp\!\Big(-\!\int_t^T r(s,X_s)\,ds\Big) \,g(X_T)\,\middle|\,\mathcal F_t\right].
\]
Fix $g_1,g_2$ and write $\Delta g:=g_1-g_2$,
$\Delta Y_t:=Y^{g_1}_t-Y^{g_2}_t$.
Then
\begin{align*}
   \Delta Y_t
=& \mathbb E\!\left[\exp\!\Big(-\!\int_t^T r(s,X_s)\,ds\Big)\,\Delta g(X_T)\,\middle|\,\mathcal F_t\right]\\
=& \exp\!\Big(\!\int_0^t r(s,X_s)\,ds\Big)\,\mathbb E\!\left[\exp\!\Big(-\!\int_0^T r(s,X_s)\,ds\Big)\Delta g(X_T)\,\middle|\,\mathcal F_t\right]\\
= & \exp\!\Big(\!\int_0^t r(s,X_s)\,ds\Big)\, M_t,
\end{align*}
where $M_t:=\mathbb E\!\left[\exp\!\Big(-\!\int_0^T r(s,X_s)\,ds\Big)\Delta g(X_T)\,\middle|\,\mathcal F_t\right]$ is a square-integrable
martingale.

Because $0\le r\le\bar r$, we have
\(
\sup_{0\le t\le T} \exp\!\big(\!\int_0^t r(s,X_s)\,ds\big)
\le e^{\bar r T}.
\)
Therefore
\[
\sup_{0\le t\le T}|\Delta Y_t|
\le e^{\bar r T}\,\sup_{0\le t\le T}|M_t|.
\]
Taking expectations and applying Doob's inequality for martingales,
\begin{align*}
 \mathbb E\!\left[\sup_{0\le t\le T}|\Delta Y_t|^2\right]
\le & e^{2\bar r T}\,\mathbb E\!\left[\sup_{0\le t\le T}|M_t|^2\right] \\
\le & 4 e^{2\bar r T}\,\mathbb E\!\left[|M_T|^2\right] = 4 e^{2\bar r T}\,\mathbb E\!\left[\exp\!\Big(\!\int_0^T -2r(s,X_s)\,ds\Big)\, |\Delta g(X_T)|^2\right].
\end{align*}
Since $r\ge0$, $\exp\!\Big(\!\int_0^T -2r(s,X_s)\,ds\Big) \le 1$ a.s., hence
\[
\mathbb E\!\left[\sup_{0\le t\le T}|\Delta Y_t|^2\right]
\le 4 e^{2\bar r T}\,\mathbb E\!\left[|\Delta g(X_T)|^2\right]
\le 4 e^{2\bar r T}\,\mathbb E\!\left[\sup_{0\le t\le T}|\Delta g(\red{X_t})|^2\right].
\]
\end{proof}
Next we verify the tail probability Assumption \ref{assumption_X: tail}.  According to
 \cite{azencott2006formule} and \cite{castell1993asymptotic}, we have the following estimate.
\begin{proposition}
    The solution $X_t$ \red{of}
   the equation \eqref{sde model}
    has the following tail probability\red{ estimate,}
\begin{align}
\mathbb P(\sup_{t\in[0,T]}| X_t -x_0|\ge r )\le \exp(-\frac{c r^{2}}{T}),
\end{align}
for some \red{constant} $c>0$\red{ and any} $r>0$.
\end{proposition}
\begin{remark}
   For SDEs driven by fractional Brownian motion, such estimates are proved in \cite{baudoin2016probability}.
For rough differential equations, the corresponding estimates are established in \cite{feng2020taylor}.
\end{remark}

\section{Deep neural operator for American option pricing and PDE with free boundary}\label{sec: pricing american}
Given a Markov process $X_t$ and exercise payoff $g(t, X_t)$, the price of an American option is the value function
\begin{equation}
\label{american option}
u(t,x) = \sup_{\tau \in \mathcal{T}_{t,T}}
\mathbb{E}^{\mathbb{Q}}\!\left[
\exp\!\left(-\int_t^\tau r(s, X_s)\,ds\right) \, g(\tau, X_\tau)
\Big| X_t = x
\right],
\end{equation}
where $\mathcal{T}_{t,T}$ is the set of stopping times with values in $[t,T]$. Thanks to \cite{el1997reflected}, the triple $(Y,Z,K)$ satisfies the \emph{reflected backward SDE}:
\begin{equation}
\begin{aligned}
Y_t &= g(T,X_T) - \int_t^T r(s,X_s)Y_s\,ds
     + K_T - K_t - \int_t^T Z_s\,dW_s, \\
Y_t &\ge g(t,X_t),\quad
\int_0^T   (Y_s - g(s,X_s))\,dK_s = 0,
\end{aligned}
\label{american FBSDE}
\end{equation}
with $(Y,Z,K) \in \mathcal{S}^2 \times \mathcal{H}^2 \times \mathcal{A}^2$. Here $\cS^2$ denotes \red{the space of} square-integrable adapted processes, $\cH^2$ \red{the space of}
 predictable processes with square-integrable norm, and  $\cA^2$
 \red{the space of} increasing, adapted, square-integrable processes vanishing at 0.


The price $u(t,x)$ can also be formulated using \red{a} variational inequality. If $u(t,x)$ is sufficiently smooth, it satisfies the obstacle problem:
\begin{equation}
\max\!\left\{\partial_t u + \mathcal{L}u - r u + c, \, g - u\right\} = 0,
\qquad u(T,x) = g(T,x).
\end{equation}
Our \red{results} thus also provide a deep neural operator approximation for \red{PDEs with free boundaries}.
For the special case of \red{the} Black--Scholes American \red{option on} a single asset $\red{X_t}$ \red{with the dynamics}
\[
dX_t = X_t(r\,dt + \sigma\,dW_t)\red{,}
\]
\red{the} corresponding PDE becomes
\begin{equation}
\max\!\left\{
\partial_t u + \frac{1}{2}\sigma^2 x^2 u_{xx} + rx u_x - r u, \;
g(x) - u
\right\} = 0, \qquad u(T,x) = g(x).\label{PDE free}
\end{equation}
In what follows, we denote by $\Gamma^A$
  the American option pricing operator associated with \eqref{american option} and \eqref{american FBSDE}, \red{or} equivalently \red{with} the free-boundary \red{PDE} problem \eqref{PDE free}. We first show that $\Gamma^A$  satisfies a Lipschitz condition.

\begin{theorem}\label{them: american bound}
Assume that $\mathbb E\Big[\Big|\sup_{0\le t\le T}g(X_t)\Big|^2 \Big]<\infty$\red{. Then} we have
\begin{equation}
    \mathbb E\left[
\sup_{0\le t\le T} |\Gamma^A(g_1)_t-\Gamma^A(g_2)_t|^2\right] \le L_{\Gamma } \mathbb E\left[ \|g_1(X_T)-g_2(X_T)\|^2\right]+C_{f,g_1,g_2} (\mathbb E\left[|(g_1-g_2)^*_T|^2\right])^{1/2},\nonumber
    \end{equation}
    where \red{$\Gamma^A(g_i)_t=Y^i_t$ denotes} the solution \red{of} \eqref{american FBSDE} with terminal and \red{obstacle} function $g_i$.
\end{theorem}
\begin{proof}
    According to \red{\cite[Theorem 6.2.3]{zhang2017backward}} with \red{the} same generator $f(Y,Z)=rY$ in the BSDE \eqref{american FBSDE}, we first have the following \red{estimate},
\begin{equation}
\begin{split}
  \mathbb E\left[
\sup_{0\le t\le T} |\Gamma^A(g_1)_t-\Gamma^A(g_2)_t|^2\right] &\le L_{\Gamma } \mathbb E[ \|g_1(X_T)-g_2(X_T)\|^2] \\
&+C( I_1+I_2 )\mathbb E[(\sup_{0\le t\le T}|g_1(X_t)-g_2(X_t)| )^2 ]^{1/2}\\
&\le [L_{\Gamma }+C( I_1+I_2 )](\mathbb E[(\sup_{0\le t\le T}|g_1(X_t)-g_2(X_t)| )^2 ])^{1/2},\nonumber
\end{split}
\end{equation}
where
\begin{equation}
    I_i:= \mathbb E[|g_i(X_T)|^2+(\int_0^T |f^i_t(0,0)|dt)^2+|\red{\sup_{0\le t\le T}(g_i(X_t))^+}|^2].
\end{equation}
\red{Denoting by $C_{f,g_1,g_2}$} the constant depending on $g_1,g_2,L_\Gamma,f$, we conclude the proof.
\end{proof}

\subsection{Operator approximation for American option pricing operator}
In this section, we generalize the operator approximation framework from Section \ref{sec: universal} to encompass a wider class of operators. The extension is based on a Lipschitz assumption motivated by Theorem \ref{them: american bound}, which naturally arises in the study of American option pricing problems, reflected FBSDEs, and PDEs with free boundary conditions.

\begin{assumption}\label{assumption_G: american}
    Assume the operator
\[
\Gamma^A : \cG  \longrightarrow \mathcal U,
\qquad g \longmapsto u = \Gamma^A( g),
\] from $\cG$ to $\cU$ is Lipschitz\red{, i.e.,} there exists $L_{\Gamma^A}\red{>0}$ such that for any $g_1,g_2\in \cG$, we have
    \begin{align*}
&\mathbb E\!\left[
  \sup_{0 \le t \le T}
  |\Gamma^A (g_1) (X_t) - \Gamma^A (g_2)(X_t)|^2
\right]\\
\le&
L_{\Gamma^A}^2\left( \,
\mathbb E\!\left[
  \sup_{0 \le t \le T}
  |g_1(X_t) - g_2(X_t)|^2
\right]+\left(\mathbb E\!\left[
  \sup_{0 \le t \le T}
  |g_1(X_t) - g_2(X_t)|^2
\right]\right)^{1/2} \right)\red{,}
\end{align*}
\red{or} equivalently,
\begin{align}
    \|\Gamma^A(g_1) - \Gamma^A(g_2)\|_{S^2}^2 \leq L_{\Gamma^A}^2 (\|g_1 - g_2\|_{S^2}+\|g_1 - g_2\|_{S^2}^2).
\end{align}
For \red{notational} simplicity, we denote $\|g_1-g_2\|_{S^2}^{1,2}:=\|g_1 - g_2\|_{S^2}+\|g_1 - g_2\|_{S^2}^2$.
\end{assumption}

We next prove the operator approximation under Assumption \ref{assumption_G: american}.

\begin{theorem}[American Option Pricing Operator]\label{thm_operator_american}
Let Assumptions \ref{assumption_X: power}, \ref{assumption_X: tail}, \ref{assumption: polynomial}, \ref{assumption_input}, and \ref{assumption_G: american} hold.
    For any $\varepsilon>0$,
    set \textcolor{black}{$N_1 = O\left(\varepsilon^{-2d_2}\right)$}, and $N^{\delta} = O(\varepsilon^{-4d_1d_2-2d_1})$.
    Define the network architecture $\cF_1=\cF_{\rm NN}(d_2,1,\red{\cL_1,\mathfrak p_1},K_1,\kappa_1, R_1)$ and $\cF_2=\cF_{\rm NN}(N^{\delta},1,\red{\cL_2,\mathfrak p_2},K_2,\kappa_2, R_2)$ with
\begin{align*}
\textcolor{black}{\cL_1 = O\left((\frac{1}{2}d_2+ 2d_2^2)\log(\varepsilon^{-2}) \right)}, \mathfrak p_1 = O(1),\ K_1 = O\left((\frac{1}{2}d_2+ 2d_2^2)\log(\varepsilon^{-2}) \right),
\kappa_1=O\left(\varepsilon^{-2d_2} \right), R_1 = 1\red{,}
\end{align*}
and
\begin{align*}
        \mathcal L_2 &= O\left(
        (N^{\delta}+\frac{1}{2}(N^{\delta})^2 )\log(N^{\delta}) + (2(N^{\delta})^2+N^{\delta}) \log\varepsilon^{-(2d_2+1)}+ N^{\delta}\log(\tilde{r})
        \right), \\
\mathfrak p_2 &= \textcolor{black}{O(\sqrt{N^{\delta}}^{N^\delta}  \varepsilon^{-(4d_1-2) N^{\delta} } )} , \\
        K_2 &= O\left(
         (N^{\delta}+\frac{1}{2}(N^{\delta})^2 )\log(N^{\delta}) + (2(N^{\delta})^2+N^{\delta}) \log\varepsilon^{-2(d_2+1)}+ N^{\delta}\log(\tilde{r})
        \right), \\
\kappa_2 &=O(\varepsilon^{-N^{\delta}(2d_2+1) } ), R_2 = 1,
\end{align*}
where $\tilde{r}$ is a constant.
For any operator  $\Gamma^A:\mathcal G\rightarrow \mathcal U$  satisfying Assumption \ref{assumption_G: american}, there are $\{\widetilde{q}_k\}_{k=1}^{N_1}$ with $\widetilde{q}_k  \in \mathcal{F}_1$ and $\{\widetilde{a}_k\}_{k=1}^{N_1} $ with $\widetilde{a}_k \in \mathcal{F}_2$ such that
\begin{align}
        \sup_{g\in \mathcal G} \E\bigg[ \sup_{0 \leq t \leq T} \left|\Gamma^{A} (g)(X_t)-\sum_{k=1}^{N_1} \widetilde{a}_k(\bg )\widetilde{q}_k(X_t)\right| \bigg] \leq \varepsilon.
    \end{align}
\end{theorem}

\begin{proof}[Proof of Theorem \ref{thm_operator_american}]
We reproduce the function, functional and operator \red{approximations} under the new Assumption \ref{assumption_G: american}.

\noindent{\bf Step 1 (function approximation)}: We first prove the function approximation under Assumption \ref{assumption_G: american}\red{,} which is the building block for the rest of the proof. Following the proof of Theorem \ref{thm_function}, the only difference is that we need to estimate $\mathcal I_2$ in the proof of Theorem \ref{thm_function}, which now has the following form due to the new Assumption \ref{assumption_G: american}.
\red{Specifically}, the function approximation for the function $\Gamma^A(g)$ following the proof of Theorem \ref{thm_function} has the following form,
\begin{align*}
   & \mathbb E\Big[\sup_{0\le t\le T} \Big|\Gamma^A(g)(X_t)-\sum_{k=1}^{N^{d_2}} \Gamma^A(g)(\cbb_k)\phi_{\cbb_k}(X_t) \Big|^2 \Big]\\
   =&  \underbrace{ \mathbb E\Big[\sup_{0\le t\le T} \Big|\Gamma^A(g)(X_t) \Big|^2 \mathbf{1}_{\Omega_r^C}(X_t)\Big]}_{\mathcal I_1^A}+\underbrace{\mathbb E\Big[\sup_{0\le t\le T} \Big|\Gamma^A(g)(X_t)-\sum_{k=1}^{N^{d_2}} \Gamma^A(g)(\cbb_k)\phi_{\cbb_k}(X_t) \Big|^2\mathbf{1}_{\Omega_r}(X_t) \Big]}_{\mathcal I_2^A}.
\end{align*}
The estimate of $\mathcal I_1^A$ \red{is the same as} in \eqref{est: I1}. For the second term $\mathcal{I}_2^A$, reproducing the estimates in \eqref{est: I2} under Assumption \ref{assumption_G: american} for the function $\Gamma^A(g)\in \mathcal U$, we have
\begin{align}
  \mathcal I_2^A =& \mathbb E\Big[\sup_{0\le t\le T} \Big|\sum_{k=1}^{N^{d_2}} [\Gamma^A(g)(X_t)-\Gamma^A(g)(\cbb_k)]\phi_{\cbb_k}(X_t) \Big|^2 \mathbf{1}_{\Omega_r}(X_t) \Big]\nonumber\\
  \leq &\mathbb E\Big[\sup_{0\le t\le T} \Big(\sum_{k=1}^{N^{d_2}} |\Gamma^A(g)(X_t)-\Gamma^A(g)(\cbb_k)||\phi_{\cbb_k}(X_t)|\Big)^2 \mathbf{1}_{\Omega_r}(X_t)\Big] \nonumber \\
  = & \mathbb E\Big[\sup_{0\le t\le T} \Big( \sum_{k: \|\cbb_k- X_t\|_{\infty} \leq \frac{2r}{(N-1)}}  |\Gamma^A(g)(X_t)-\Gamma^A(g)(\cbb_k)|\phi_{\cbb_k}(X_t)\Big)^2 \mathbf{1}_{\Omega_r}(X_t)\Big] , \nonumber\\
  \leq & \mathbb E\Big[\sup_{0\le t\le T} \Big( \max_{k: \|\cbb_k-X_t\|_{\infty} \leq \frac{2r}{(N-1)}}  |\Gamma^A(g)(X_t)-\Gamma^A(g)(\cbb_k)|\Big(\sum_{k: \|\cbb_k-X_t\|_{\infty} \leq \frac{2r}{(N-1)}}  \phi_{\cbb_k}(X_t) \Big)\Big)^2\mathbf{1}_{\Omega_r}(X_t)\Big] \nonumber \\
   \leq & \mathbb E\Big[\sup_{0\le t\le T} \max_{k: \|\cbb_k-X_t\|_{\infty} \leq \frac{2r}{(N-1)}}  |\Gamma^A(g)(X_t)-\Gamma^A(g)(\cbb_k)|^2\mathbf{1}_{\Omega_r}(X_t)\Big] \nonumber \\
    \le&  L_{\Gamma^A}^2\left( \,
\mathbb E\!\left[
  \sup_{0 \le t \le T}
  \max_{k: \|\cbb_k-X_t\|_{\infty} \leq \frac{2r}{(N-1)}}|g(X_t) - g(\cbb_k)|^2
\right]\right.\nonumber \\
&\quad \left.+\left(\mathbb E\!\left[
  \sup_{0 \le t \le T}
  \max_{k: \|\cbb_k-X_t\|_{\infty} \leq \frac{2r}{(N-1)}}|g(X_t) - g(\cbb_k)|^2
\right]\right)^{1/2} \right)\nonumber \\
\le &L_{\Gamma^A}^2\left(\frac{(2rL_g)^2 d_2 }{(N-1)^2}+\left(\frac{(2rL_g)^2 d_2 }{(N-1)^2}\right)^{1/2} \right)  \nonumber\\
\le & CL_{\Gamma^A}^2\frac{2rL_gd_2^{1/2} }{(N-1)} \leq \frac{1}{4}\times \frac{\varepsilon_1^2 }{4},\label{epsilon square est}
\end{align}
where we use the Lipschitz assumption \red{on} $\Gamma^A$ from Assumption \red{\ref{assumption_G: american}}, and $\left(\frac{(2rL_g)^2 d_2 }{(N-1)^2}\right)^{1/2}$ is the dominating term since $\frac{(2rL_g)^2 d_2 }{(N-1)^2}<1$ by choosing $\varepsilon_1$ to be sufficiently small. Following the same idea and estimate in \eqref{est: I2-2}, in order to get the desired estimate in \eqref{function S 12 A}, we make the following selection.

Up to a \red{change of the constant $c$}, for $\varepsilon_1>0$ \red{to be specified later},
    set $r=\left\lceil -\frac{2C_T}{c}\log \frac{\varepsilon_1^2}{16C_0}\right\rceil^{\frac{1}{\alpha}}$, and define $Q_r, \Omega_r$ as in Definition \ref{def: omega r}.
    \red{Following the estimates in \eqref{epsilon square est}, there exist} a constant $N = O\left(rL_gd_1^{\frac{1}{2}}\varepsilon_1^{-2}\right)$ \textcolor{black}{and $N_1 = CN^{d_2}$ for some constant $C$} depending on $d_2, L_g$ and $r$, a network architecture $\mathcal{F}_1=\cF_{\rm NN}(d_2, 1 , \cL_1, \mathfrak p_1, K_1, \kappa_1, R_1)$\red{, networks}  $\{\widetilde{q}_k\}_{k = 1}^{N_1}$ with $\widetilde{q}_k\in \mathcal{F}_1$, and \red{points} $\{\cbb_k\}_{k = 1}^{N_1} \subset Q_r$ such that for any $g\in \mathcal G$\red{ the following two estimates hold. First,}
    \begin{align}
       \left(  \E \bigg[ \sup_{0 \leq t \leq T} \left|\Gamma^A (g)(X_t)-\sum_{k = 1}^{N_1} \Gamma^A (g)(\cbb_k) \widetilde{q}_k(X_t)\right|^2 \bigg]\right)^{1/2}\leq  \varepsilon_1/2.
        \label{eq.operator.1A}
    \end{align}
 Since $\varepsilon_1/2<1$, this further implies
    \begin{align}
        \E \bigg[ \sup_{0 \leq t \leq T} \left|\Gamma^A (g)(X_t)-\sum_{k = 1}^{N_1} \Gamma^A (g)(\cbb_k) \widetilde{q}_k(X_t)\right|^2 \bigg]
        \le \varepsilon_1^2/4 \leq \varepsilon_1/2.
        \label{eq.operator.2AA}
    \end{align}
 Combining the above two estimates, we get
    \begin{align}\label{function S 12 A}
    \|\Gamma^A(g)-\sum_{k=1}^{N_1}\Gamma^A(g)(\mathbf c_k)\tilde{q}_k\|_{\mathcal{S}_2}^{1,2}\le \varepsilon_1.
    \end{align}
According to Theorem \ref{thm_function}, and the relations between \eqref{eq.operator.1A} and \eqref{eq.operator.2AA}, the network size \red{of} $\mathcal F_1$ will be determined by the estimate in \eqref{eq.operator.2AA}.
Since the Lipschitz assumption is not used in  \eqref{eqn_estimate_of_delta_in_function}, following the same estimate in \eqref{eqn_estimate_of_delta_in_function} and \red{the estimates} $r = O\left((\log(\varepsilon_1^{-2}) )^{\frac{1}{\alpha}} \right), N = O\left(rL_gd_1^{\frac{1}{2}}\varepsilon_1^{-2}\right)$,
we have $\delta = O\left(d_1^{-\frac{d_2}{2}-1}\varepsilon_1^{\frac{1}{2} + 2d_2}(\log \varepsilon_1^{-2})^{-\frac{p+d_2}{\alpha}} \right)$.
Such a network has size
\begin{align*}
\cL_1& = O\left((\frac{1}{2}d_2+ 2d_2^2)\log(\varepsilon_1^{-2}) \right), \mathfrak p_1 = O(1),\ K_1 = O\left((\frac{1}{2}d_2+ 2d_2^2)\log(\varepsilon_1^{-2}) \right),\\
\kappa_1&=O\left(\varepsilon_1^{-2d_2} \right), R_1 = 1.
\end{align*}

\noindent{\bf Step 2 (functional approximation)}: For each $k$, define the functional induced by the operator $\Gamma^A$ as follows,
\begin{align}
    \sF_{k}(\Gamma^A( g))=\Gamma^A (g)(\cbb_k).
\label{eq:fk american}
\end{align}
For any $g_1,g_2\in \cG$, \red{when} the forward process at \red{time} $t$ starts at $\cbb_k$, we have
\begin{align}
    & |\sF_{k}(\Gamma^A (g_1))- \sF_{k}(\Gamma^A  (g_2))|^2 \nonumber \\
    \leq & \mathbb E \bigg[ \sup_{0 \leq s \leq T}|\Gamma^A (g_1)(X_s)- \Gamma^A  (g_2)(X_s)|^2 \bigg] \nonumber\\
    \leq & L_{\Gamma^A}^2\left(   \mathbb E \bigg[ \sup_{0 \leq s \leq T} |g_1(X_s) - g_2(X_s)|^2 \bigg] + \left(\mathbb E \bigg[ \sup_{0 \leq s \leq T} |g_1(X_s) - g_2(X_s)|^2 \bigg] \right)^{1/2}\right)\label{eq.proof.approx.fLip american 2}\\
    = & L_{\Gamma^A}^2 \|g_1 - g_2\|_{S^2}^{1,2},\nonumber
\end{align}
where the last inequality follows from Assumption \ref{assumption_G: american}.
\red{Recalling} the function $g_w(x)$ defined in \eqref{eq:gw-def},
for the functional $\mathsf F$ induced by the American option pricing operator $\Gamma^A$, we have
\begin{align}
\bigl|\sF(g)- \sF(g_w)\bigr|^{2}
&\le L_{\Gamma^A}^2\,\|g-g_w\|_{S^2}^{1,2} \nonumber \\
&\le L_{\Gamma^A}^2\,\|g-g_w\|_{S^2}^{2}+L_{\Gamma^A}\,\|g-g_w\|_{S^2}^{1}, \nonumber
\end{align}
which is equivalent to
\begin{align}
\bigl|\sF(g)- \sF(g_w)\bigr|
\le (L_{\Gamma^A}^2\,\|g-g_w\|_{S^2}^{2}+L_{\Gamma^A}\,\|g-g_w\|_{S^2}^{1})^{1/2}, \nonumber
\end{align}
where we assume that $\red{L_{\Gamma^A}^2\le L_{\Gamma^A}}$ to ease the computation\red{, which}, up to a \red{change of constants}, does not change the order of the radius size. (\red{The case $L_{\Gamma^A}\le L_{\Gamma^A}^2$ leads to} similar computations\red{.})
Following the idea \red{of} the derivation of \eqref{function S 12 A} from \eqref{eq.operator.1A} and \eqref{eq.operator.2AA}, up to a constant, for any $\varepsilon_2>0$, we can pick $r$ such that $L_{\Gamma^A}^2\,\|g-g_w\|_{S^2}^{2}\le \frac{\varepsilon_2^4}{64}\le \frac{\varepsilon_2^2}{8}$, which thus implies $L_{\Gamma^A}\|g-g_w\|_{S^2}^{1}\le \frac{\varepsilon_2^2}{8}$, and the following estimate
\begin{align}
\bigl|\sF(g)- \sF(g_w)\bigr|
 &\le( \frac{\varepsilon_2^2}{8}+\frac{\varepsilon_2^2}{8} )^{1/2} \le  \varepsilon_2 / 2. \label{1/4 varepsilon 2}
\end{align}

Thus, \red{the requirement} $L_{\Gamma^A}^2\,\|g-g_w\|_{S^2}^{2}\le \frac{\varepsilon_2^4}{64}$ determines the following parameters\red{:} following \eqref{est: J1}, \eqref{est: J2} and \eqref{est: F 1}, we have
\begin{align}
\label{eqn_functional_r_delta_American}
    r_{2} =\left\lceil -\frac{2C_T}{c}\log \frac{\varepsilon_2^4}{128C_0L_\sF^2}\right\rceil^{\frac{1}{\alpha}}+1, \quad \delta_2 = \frac{C\varepsilon_2^2}{8\sqrt{2}L_\sF L_gd_1^{\frac{1}{2}} },
\end{align}
where $C$ is a constant.
Now, for any $g,\tilde g\in\mathcal{G}$, define $g_w$ and $\tilde g_w$ as in \eqref{eq:gw-def}, and set
\begin{equation*}
\bg=\bigl[g(c_1),\ldots,g(c_{N^{\delta}})\bigr]^{\mathsf T},
\qquad
\tilde \bg=\bigl[\tilde g(c_1),\ldots,\tilde g(c_{N^{\delta}})\bigr]^{\mathsf T}.
\end{equation*}
Define the function $h(\bg):= \sF(g_w)$. According to \eqref{eq:fk american} and \eqref{eq.proof.approx.fLip american 2}, we have
\begin{align}
|h(\bg)-h(\tilde \bg)|^{2}
&= \bigl| \sF(g_w)- \sF(\tilde g_w) \bigr|^{2}\nonumber\\
& \le L_{\Gamma^A}^2\,\|g_w-\tilde g_w\|_{S^2}^{1,2}\nonumber\\
&= L_{\Gamma^A}^2\,\mathbb{E}\!\left[ \sup_{t\le s\le T} \!\!
\Bigl( \sum_{k=1}^{N^\delta} \bigl(g(c_k)-\tilde g(c_k)\bigr)\, w_k(X_s) \Bigr)^{2} \right]\nonumber\\
&+ L_{\Gamma^A}^2\,\left( \mathbb{E}\!\left[ \sup_{t\le s\le T} \!\!
\Bigl( \sum_{k=1}^{N^\delta} \bigl(g(c_k)-\tilde g(c_k)\bigr)\, w_k(X_s) \Bigr)^{2} \right]\right)^{1/2}\nonumber\\
&\le L_{\Gamma^A}^2  \,\Bigl( \sum_{k=1}^{N^\delta} |g(c_k)-\tilde g(c_k)|^{2} \Bigr)\,
\mathbb{E}\!\left[ \sup_{t\le s\le T} \sum_{k=1}^{N^\delta} w_k(X_s)^{2} \right]\nonumber \\
&+ L_{\Gamma^A}^2  \,\Bigl( \sum_{k=1}^{N^\delta} |g(c_k)-\tilde g(c_k)|^{2} \Bigr)^{1/2}\,
\mathbb{E}\!\left[ \sup_{t\le s\le T} \sum_{k=1}^{N^\delta} w_k(X_s)^{2} \right]^{1/2}\nonumber \\
&\le L_{\Gamma^A}^2  \,\,\Bigl( \sum_{k=1}^{N^\delta} |g(c_k)-\tilde g(c_k)|^{2} +\left( \sum_{k=1}^{N^\delta} |g(c_k)-\tilde g(c_k)|^{2}\right)^{1/2} \Bigr)\nonumber\\
&\le L_{\Gamma^A}^2  (|\bg - \tilde{\bg}|^2 +|\bg - \tilde{\bg}|),
\label{lip est h American}
\end{align}
which is equivalent to
\begin{align*}
    |h(\bg)-h(\tilde{\bg})|\le L_{\Gamma^A}(|\bg - \tilde{\bg}|^2 +|\bg - \tilde{\bg}|)^{1/2}.
\end{align*}
By Assumption \ref{assumption: polynomial}, we have the bound of \red{the} function $g$ on $\red{Q_{r_2}}$ as below,
\[g(x)\le C_g(1+|x|^p)\le C_g(1+|r_2|^p):= \tilde{r}, \quad x\in Q_{r_2}.\]
Similarly, according to our definition in \eqref{eq:fk american}, $h$ is also bounded above.
Applying Lemma \ref{lemma_new_function_approx} (which is proved below), we claim that for $\frac{1}{2} \varepsilon_2>0$
the function (functional) $h(\bg)= \sF(g_w)$ can be approximated by a network $\cF_2=\cF_{\rm NN}(N^{\delta}, 1, \mathcal{L}_2, \mathfrak{p}_2, K_2, \kappa_2, R_2)$,
\begin{align*}
        \mathcal L_2 &= O\left(
        (N^{\delta}+\frac{1}{2}(N^{\delta})^2 )\log(N^{\delta}) + (2(N^{\delta})^2+N^{\delta}) \log\varepsilon_2^{-1}+ N^{\delta}\log(\tilde{r})
        \right),  \mathfrak p_2 = O(1), \\
        K_2 &= O\left(
         (N^{\delta}+\frac{1}{2}(N^{\delta})^2 )\log(N^{\delta}) + (2(N^{\delta})^2+N^{\delta}) \log\varepsilon_2^{-1}+ N^{\delta}\log(\tilde{r})
        \right), \\
\kappa_2 &=O(\varepsilon_2^{-N^{\delta}} ), R_2 = 1.
\end{align*}
That is\red{,}
\begin{align}
    \sup_{g\in \red{\cG}} \left|\sF(g_w) - \sum_{k=1}^{N_2 } h({\bg}_k )\widetilde{q}_k \right| \leq \frac{1}{2}\varepsilon_2,
    \label{eqn_final_functional_approx}
\end{align}
\textcolor{black}{where $N_2 = N^{N^{\delta}}$, with $N =  O(\sqrt{N^{\delta}} \varepsilon_2^{-2} )$. }
Combining the above estimates in \eqref{1/4 varepsilon 2} and \eqref{eqn_final_functional_approx}, such a network architecture gives a network $\widetilde{a}_k$ so that
\begin{align*}
    \sup_g |\sF_{k}(\red{\Gamma^A} (g))-\widetilde{a}_{k}(\bg)|\leq \varepsilon_2.
\end{align*}
Therefore,
\begin{align}
        &\E \bigg[\sup_{0 \leq s \leq T} \left|\sum_{k = 1}^{N_1} \sF_{k}(\red{\Gamma^A} g)\widetilde{q}_k(X_s) - \sum_{k = 1}^{N_1} \widetilde{a}_k(\bg ) \widetilde{q}_k(X_s)\right| \bigg] \nonumber\\
        =& \E \bigg[\sup_{0 \leq s \leq T} \left|\sum_{k = 1}^{N_1} \left(\sF_k(\red{\Gamma^A} g)-\widetilde{a}_k(\bg) \right) \widetilde{q}_k(X_s)\right| \bigg]  \nonumber\\
        \leq& \sum_{k = 1}^{N_1} \sup_{\bg} |\sF_k(\red{\Gamma^A}  g) -\widetilde{a}_k(\bg )| = N_1\varepsilon_2.
        \label{eq.operator.2A}
    \end{align}

\noindent{\bf Step 3 (American pricing operator approximation)}:
Applying the Cauchy–Schwarz inequality, using (\ref{function S 12 A}) and (\ref{eq.operator.2A}), we have
    \begin{align*}
        &\sup_{g\in \cG} \E \bigg[ \sup_{0 \leq s \leq T} \left|\Gamma^A g(X_s)-\sum_{k = 1}^{N_1} \widetilde{a}_k(\bg )\widetilde{q}_k(X_s)\right| \bigg] \nonumber \\
        \leq & \sup_{g\in \cG} \E \bigg[ \sup_{0 \leq s \leq T} \left|\Gamma^A (g)(X_s)-\sum_{k = 1}^{N_1} \sF_k(\red{\Gamma^A} g) \widetilde{q}_k(X_s)\right| \bigg]\\
&        + \sup_{g\in \cG} \E \bigg[\sup_{0 \leq s \leq T} \left|\sum_{k = 1}^{N_1} \sF_{k}(\Gamma^A g)\widetilde{q}_k(X_s) - \sum_{k = 1}^{N_1} \widetilde{a}_k(\bg ) \widetilde{q}_k(X_s)\right| \bigg]\\
        \leq & \sup_{g\in \cG} \E \bigg[ \sup_{0 \leq s \leq T} \left|\Gamma^A (g)(X_s)-\sum_{k = 1}^{N_1} \sF_k(\Gamma^A g) \widetilde{q}_k(X_s)\right|^2 \bigg]^{1/2}\\
&        + \sup_{g\in \cG} \E \bigg[\sup_{0 \leq s \leq T} \left|\sum_{k = 1}^{N_1} \sF_{k}(\Gamma^A g)\widetilde{q}_k(X_s) - \sum_{k = 1}^{\red{N_1}} \widetilde{a}_k(\bg ) \widetilde{q}_k(X_s)\right| \bigg]\\
        \leq &\sup_{g\in \cG} \|\Gamma^A(g)-\sum_{k=1}^{N_1}\Gamma^A(g)(\mathbf c_k)\tilde{q}_k\|_{\mathcal{S}_2}^{1,2} + \sup_{g\in \cG} \E \bigg[\sup_{0 \leq s \leq T} \left|\sum_{k = 1}^{N_1} \sF_{k}(\red{\Gamma^A} g)\widetilde{q}_k(X_s) - \sum_{k = 1}^{N_1} \widetilde{a}_k(\bg ) \widetilde{q}_k(X_s)\right| \bigg]\\
        \leq& \varepsilon_1 + N_1\varepsilon_2.
    \end{align*}
    \red{Setting} $\varepsilon_2 = \varepsilon_1/(2N_1)$ \red{and} $\varepsilon_1=\frac{\varepsilon}{2}$,
    or $\varepsilon_2 = O\left( \varepsilon^{2d_2+1} (\log \varepsilon^{-2})^{-\frac{d_2}{\alpha}} d_1^{-\frac{d_2}{2}} \right)$,
   we then have
    \begin{align*}
        &\sup_{g\in \cG} \E \bigg[ \sup_{0 \leq s \leq T} \left|\Gamma^A g(X_s)-\sum_{k = 1}^{N_1} \widetilde{a}_k(\bg )\widetilde{q}_k(X_s)\right| \bigg] \leq \varepsilon.
    \end{align*}
Consequently, the network size is estimated to be
\begin{align*}
        \mathcal L_2 &= O\left(
        (N^{\delta}+\frac{1}{2}(N^{\delta})^2 )\log(N^{\delta}) + (2(N^{\delta})^2+N^{\delta}) \log\varepsilon^{-(2d_2+1)}+ N^{\delta}\log(\tilde{r})
        \right),\\
        \mathfrak p_2 &= \textcolor{black}{O(\sqrt{N^{\delta}}^{N^{\delta}} \varepsilon^{-(4d_1-2) N^{\delta} } )}, \\
        K_2 &= O\left(
         (N^{\delta}+\frac{1}{2}(N^{\delta})^2 )\log(N^{\delta}) + (2(N^{\delta})^2+N^{\delta}) \log\varepsilon^{-2(d_2+1)}+ N^{\delta}\log(\tilde{r})
        \right), \\
\kappa_2 &=O(\varepsilon^{-N^{\delta}(2d_2+1) } ), R_2 = 1.
\end{align*}
By Lemma \ref{lem.cover.ball} and \red{\eqref{eqn_functional_r_delta_American}},
\begin{align*}
    N^{\delta} \leq C(\log \varepsilon_2^{-4})^{\frac{d_1}{\alpha}}\varepsilon_2^{-2d_1}d_1^{\frac{d_1}{2}}.
\end{align*}
\red{Substituting} back \red{in terms of} $\varepsilon$, we have $N^{\delta} = O(\varepsilon^{\red{-4d_1d_2-2d_1}})$.


\end{proof}

\begin{lemma}\label{lemma_new_function_approx}
Let $Q_{\tilde{r}} = [-\tilde{r}, \tilde{r}]^{N^{\delta}} $, $\bg=\bigl[g_1,\ldots,g_{N^{\delta}} \bigr]^{\mathsf T}\in Q_{\tilde{r}}$ and $\beta_{\bg} = \sup_{g\in Q_{\tilde{r}}}|h(g)|$.
    Assume \red{that the} function $h: Q_{\tilde{r}} \rightarrow \mathbb R$, \red{where} $N^{\delta}$ \red{is} an integer and \red{$\tilde r$ is a constant},  satisfies the following assumption\red{:}
    \begin{align}
        |h(g)-h(\tilde{g})|\le L (|g-\tilde{g}|^2+|g-\tilde{g}|)^{1/2}, \red{\quad} g, \tilde{g}\in Q_{\tilde{r}},
        \label{eqn_lemma_function_approx_american}
    \end{align}
where $L$ is a constant.
For $\varepsilon_2>0$, there exist $N_2$, $\{ \bg_k\}_{k = 1}^{N_2}\subset Q_{\tilde{r}}$, and a network $\cF_2=\cF_{\rm NN}(N^{\delta}, 1, \mathcal{L}_2, \mathfrak{p}_2, K_2, \kappa_2, R_2)$, where
\begin{align*}
        \mathcal L_2 &= O\left(
        (N^{\delta}+\frac{1}{2}(N^{\delta})^2 )\log(N^{\delta}) + (2(N^{\delta})^2+N^{\delta}) \log\varepsilon_2^{-1}+ N^{\delta}\log(\tilde{r})
        \right),  \mathfrak p_2 = O(1), \\
        K_2 &= O\left(
         (N^{\delta}+\frac{1}{2}(N^{\delta})^2 )\log(N^{\delta}) + (2(N^{\delta})^2+N^{\delta}) \log\varepsilon_2^{-1}+ N^{\delta}\log(\tilde{r})
        \right), \\
\kappa_2 &=O(\varepsilon_2^{-N^{\delta}} ), \red{R_2} = 1,
\end{align*}
such that
\begin{align*}
 \sup_{g\in Q_{\tilde{r}}} \left| h(g) - \sum_{k=1}^{N_2 } h({\bg}_k )\widetilde{q}_k \right| \leq \varepsilon_2,
\end{align*}
where $N_2 = N^{N^{\delta}}$, with $N =  O(\red{\sqrt{N^{\delta}}} \varepsilon_2^{-2} )$.
\end{lemma}
We now prove Lemma \ref{lemma_new_function_approx}\red{, which} is used to approximate the functional.
\begin{proof}[Proof of Lemma \ref{lemma_new_function_approx}]
Let us partition $Q_{\tilde{r}}$ into $N^{N^{\delta}}$ subcubes for some $N$ to be specified later, \red{where} $N^{\delta}$ \red{is as in} Theorem \ref{thm_operator_american}.
Let $\{\bg_k\}_{k=1}^{N^{N^{\delta}}}$ be a uniform grid on $Q_{\tilde{r}}$ so that $\bg_k\in \left\{-\tilde{r},-\tilde{r}+\frac{2\tilde{r}}{N-1}, ..., \tilde{r}\right\}^{N^{\delta}}$ for each $k$.
Define
\begin{align}
    \psi(a) = \begin{cases}
        1, |a|<1,\\
        0, |a|>2, \\
        2-|a|, 1\leq |a|\leq 2,
    \end{cases}
    \label{eqn_psi2}
\end{align}
with $a\in\mathbb{R}$, and
\begin{align}
    \phi_{\bg_k}(\xb) = \prod_{j = 1}^{N^{\delta}} \psi \left(\frac{3(N-1)}{2\tilde{r} }(x_j-g_{k,j})\right).
    \label{eqn_phi2}
\end{align}
In this definition, we have $\supp(\phi_{\bg_k})=\left\{\xb: \|\xb-\bg_k\|_{\infty}\leq \frac{4\tilde{r}}{3(N-1)}\right\}\subset \left\{\xb: \|\xb-\bg_k\|_{\infty}\leq \frac{2\tilde{r}}{(N-1)}\right\}$ and
$$
\max_k \phi_{\bg_k} =1, \quad \sum_{k=1}^{N^{N^{\delta}}} \phi_{\bg_k}=1.
$$
We construct a piecewise constant approximation to $h$ as
$$
\bar{h}(\xb)=\sum_{k=1}^{N^{N^{\delta}}} h(\bg_k)\phi_{\bg_k}(\xb).
$$
It follows that
\begin{align}
    |h(\bg)-\bar{h}(\bg)|=&\left| \sum_{k=1}^{N^{N^{\delta}}} \phi_{\bg_k}(\red{\bg})(h(\bg)-h(\bg_k))\right| \nonumber\\
    \leq & \sum_{k=1}^{N^{N^{\delta}}} \phi_{\bg_k}(\red{\bg})|h(\bg)-h(\bg_k)| \nonumber\\
    = &\sum_{k: \|\bg_k-\red{\bg}\|_{\infty} \leq \frac{2\tilde{r} }{(N-1)}}  \phi_{\bg_k}(\red{\bg})|(h(\red{\bg})-h(\bg_k))| \nonumber\\
    \leq & \max_{k: \|\bg_k-\red{\bg}\|_{\infty}\leq \frac{2\tilde{r} }{(N-1)}}|(h(\red{\bg})-h(\bg_k))|\left(\sum_{k: \|\bg_k-\red{\bg}\|_{\infty} \leq \frac{2\tilde{r} }{(N-1)}}  \phi_{\bg_k}(\red{\bg})\right) \nonumber\\
    \leq& \max_{k: \|\bg_k-\red{\bg}\|_{\infty}\leq \frac{2\tilde{r} }{(N-1)}}|(h(\red{\bg})-h(\bg_k))| \nonumber\\
    \leq & \red{L}\left(\left(\frac{2\sqrt{N^{\delta}}\tilde{r} }{N-1}\right)^2 + \left(\frac{2\sqrt{N^{\delta}}\tilde{r}}{N-1} \right)\right)^{\frac{1}{2}}<\frac{\varepsilon_2}{2},
\end{align}
where the \red{second-to-last inequality} follows from \red{assumption \eqref{eqn_lemma_function_approx_american}, and the last inequality follows by setting} $N = C \red{L}^2\varepsilon_2^{-2} (N^{\delta})^{\frac{1}{2}}\tilde{r}$,
where $C$ is a constant.
Now we adopt \red{neural networks} $\tilde{q}_k\in\mathcal{F}_2$\red{, with structure} to be specified later\red{,} to approximate $\phi_{{\bg}_k}$ such that $\|\phi_{{\bg}_k} - \widetilde{q}_{k}\|_{L^{\infty}(Q_{\tilde{r}})}\leq N^{\delta}\tilde{\delta}$, with $\tilde{\delta}$ to be specified later.
We hence have
    \begin{align}
        \left\| \sum_{k=1}^{N^{N^{\delta}}}
        h({\bg}_k )\widetilde{q}_k - \bar{h} \right\|_{L^{\infty}(Q_{\tilde{r}}) }
        \leq & \sum_{k=1}^{N^{N^{\delta}}}|h({\bg}_k)|\|\widetilde{q}_k-\phi_{{\bg}_k}\|_{L^{\infty}(Q_{\tilde{r}}) } \nonumber \\
        \leq &  N^{\delta}N^{N^{\delta}}\beta_{\bg}\tilde{\delta}.
    \end{align}
Now set $\tilde{\delta} = C\frac{\varepsilon_2}{N^{\delta}N^{N^{\delta}}\beta_{\bg}}$, where $C$ is a constant.
\red{Consequently, the} network has \red{the} architecture $\cF_2=\cF_{\rm NN}(N^{\delta}, 1, \mathcal{L}_2, \mathfrak{p}_2, K_2, \kappa_2, R_2)$, where
\begin{align*}
        \mathcal L_2 &= O\left(
        (N^{\delta}+\frac{1}{2}(N^{\delta})^2 )\log(N^{\delta}) + (2(N^{\delta})^2+N^{\delta}) \log\varepsilon_2^{-1}+ N^{\delta}\log(\tilde{r})
        \right),  \mathfrak p_2 = O(1), \\
        K_2 &= O\left(
         (N^{\delta}+\frac{1}{2}(N^{\delta})^2 )\log(N^{\delta}) + (2(N^{\delta})^2+N^{\delta}) \log\varepsilon_2^{-1}+ N^{\delta}\log(\tilde{r})
        \right), \\
\kappa_2 &=O(\varepsilon_2^{-N^{\delta}} ), R_2 = 1,
\end{align*}
\red{and} the hidden constant depends on $\red{L}$ and $\beta_{\bg}$.
\end{proof}

\section{Algorithm}\label{sec: numerics}
A unifying view \red{for solving} a basket of American options has been studied in
\cite{bank2003american} by using one unified Snell envelope. In this section, we use the proposed deep neural operator to solve this \red{basket-of-American-options} problem. In particular, we are able to provide the exercise boundary for new strike prices within the range of our training \red{set}. The precise training process and model specification \red{are} presented below.

In the following numerical example, we train a deep \red{neural operator} to obtain the pricing operator of \red{a} Bermudan put option. Then we plot the exercise boundaries for various terminal payoff functions in Figure \ref{fig:put}.

The ground-truth training data are produced by a fully implicit finite-difference discretization of the Black--Scholes American pricing PDE in log-price variables on a uniform grid, closely following standard references on PDE methods for options \cite{WilmottHowisonDewynne1995}.

Under the \red{risk-neutral} measure, the price of an American option with strike \(K\), volatility \(\sigma>0\), and \red{risk-free} rate \(r > 0\) satisfies the Black–Scholes PDE with free boundary as in \eqref{PDE free}
with terminal condition $u(T,x) = g_k(x) = \max(K-x,0). $

The uniform log-price grid is given by \(y_{\min}=\ln x_{\min}\), \(y_{\max}=\ln x_{\max}\), and \(y_j=y_{\min}+j\Delta y\) for \(j=0,\dots,N_x-1\), where \(\Delta y=(y_{\max}-y_{\min})/(N_x-1)\). The corresponding stock prices are \(x_j=e^{y_j}\).
The time grid is given by \(t_n=n\Delta t\) for \(n=0,\dots,N_t\), with \(\Delta t=T/N_t\). In our numerical example, we set the risk-free interest rate \( r = 0.1 \), the volatility \( \sigma = 0.2 \),  the time to maturity \( T = 1 \),  $N_t = 50$ and $N_x = 300$. The price interval is chosen wide enough to contain the early-exercise region for all strikes in the training range: with $K_{\min}=90$ and $K_{\max}=120$, we take $x_{\min}=K_{\min}/2=45$ and $x_{\max}=1.5K_{\max}=180$.

We first perform a log transformation\red{: let} \(y=\ln x\) and define \(v(y,t):=u(x,t)=u(e^{y},t)\). Then
\[
u_x = \frac{1}{x}v_y,\qquad
u_{xx} = \frac{1}{x^2}\left(v_{yy}-v_y\right),
\]
and \eqref{PDE free} becomes the \red{constant-coefficient convection--diffusion} equation
\begin{equation}
\label{eq:log-PDE}
\partial_t v + \frac{1}{2}\sigma^2 v_{yy} + \mu\, v_y - r v = 0,
\qquad \mu := r - \tfrac12\sigma^2.
\end{equation}
Denote \(v_j^n \approx v(y_j,t_n)\).
Use centered differences in space at time level \(n\):
\[
v_y(y_j,t_n) \approx \frac{v_{j+1}^n - v_{j-1}^n}{2\Delta y},\qquad
v_{yy}(y_j,t_n) \approx \frac{v_{j+1}^n - 2v_{j}^n + v_{j-1}^n}{\Delta y^2}.
\]
The finite difference algorithm runs backward in time from \(t_{N_t}=T\) to \(t_0=0\), \red{with} the fully implicit step
\begin{equation}
\label{eq:BE-stencil}
\frac{v_j^{n+1}-v_j^{n}}{\Delta t}
+ \frac12\sigma^2 \frac{v_{j+1}^{n}-2v_j^{n}+v_{j-1}^{n}}{\red{\Delta y^2}}
+ \mu \frac{v_{j+1}^{n}-v_{j-1}^{n}}{\red{2\Delta y}}
- r v_j^{n} = 0.
\end{equation}
After rearranging into matrix form, and enforcing the free boundary condition, $(v)^n $ \red{is} obtained from $(v)^{n+1}$.

Our approach uses operator learning for Bermudan-style put options. From sampled space--time values of option prices, we train a neural operator that, given a payoff function, reconstructs the full price surface and thereby recovers the optimal exercise (stopping) boundary.

Let \(\Gamma\) denote the pricing operator that maps a payoff \(g_k\) to its price surface \(u_k\). On a grid \(\{x_j\}_{j=0}^{N_x-1}\times\{t_n\}_{n=1}^{N_t}\),
\[
u_k(x_j,t_n)=\Gamma g_k(x_j,t_n).
\]
The neural operator $\Gamma_{\theta}$ with parameters $\theta$ is trained to approximate the pricing operator by minimizing the empirical mean-squared error
\begin{equation}\label{eqn:obj}
\mathcal{L}(\theta)
= \frac{1}{K\,N_t\,N_x}\sum_{k=1}^{K}\sum_{j=0}^{N_x-1}\sum_{n=1}^{N_t}
\bigl[\Gamma_\theta g_k (x_j,t_n)-u_k(x_j,t_n)\bigr]^2.
\end{equation}
This optimization ensures that the operator network learns an accurate mapping from input payoff functions to their corresponding option price surfaces. Our algorithm is capable of computing the exercise boundary for any strike price between 90 and 120. We select six payoff functions with different strike prices and present their corresponding exercise boundaries in Figure \ref{fig:put}. Once trained, the learned operator \( \Gamma_\theta \) can be utilized to recover the entire exercise boundary from the approximated solution surface.

\begin{figure*}[htb]
    \centering 
\begin{minipage}[t]{.5\textwidth}
\begin{subfigure}{\textwidth}
  \includegraphics[width=\linewidth]{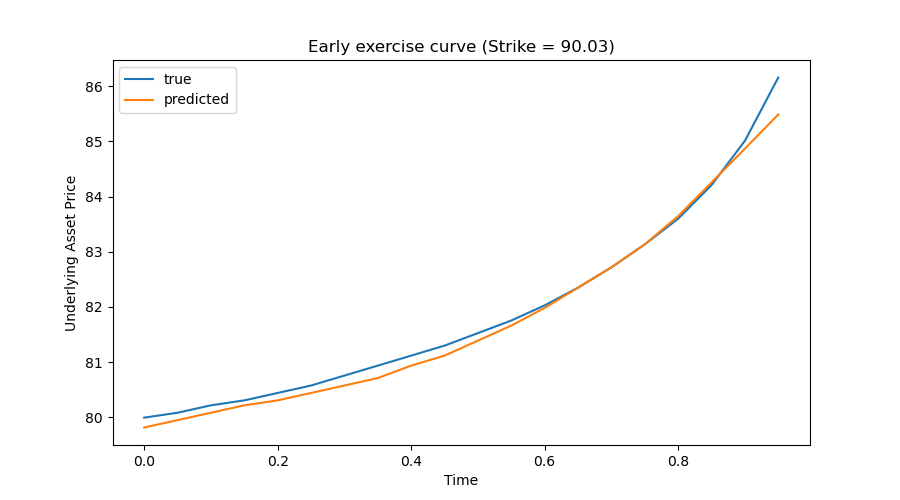}
\end{subfigure}\hfil 
\begin{subfigure}{\textwidth}
  \includegraphics[width=\linewidth]{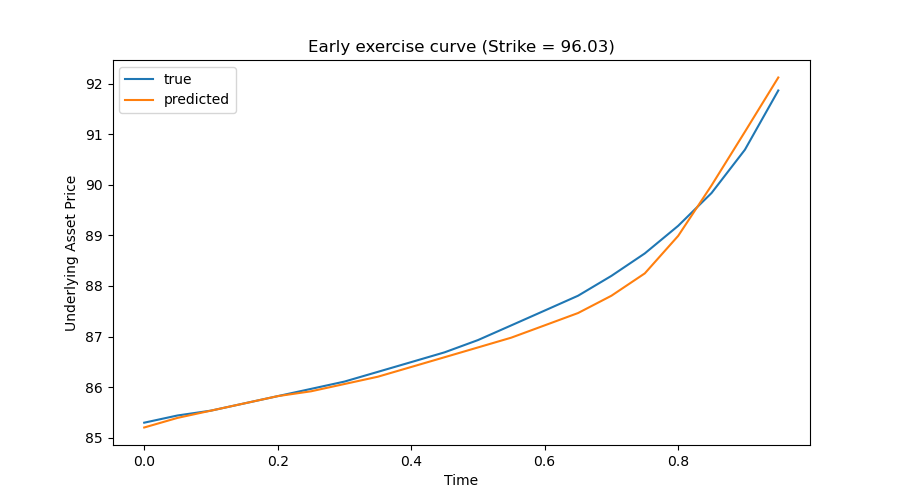}
\end{subfigure}\hfil 
\end{minipage}\hfil
\begin{minipage}[t]{.5\textwidth}
\begin{subfigure}{\textwidth}
  \includegraphics[width=\linewidth]{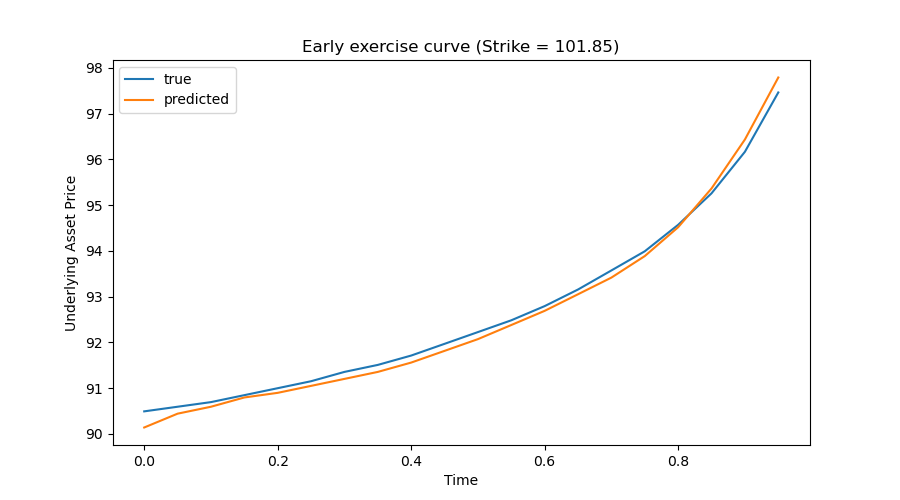}
\end{subfigure}

\begin{subfigure}{\textwidth}
  \includegraphics[width=\linewidth]{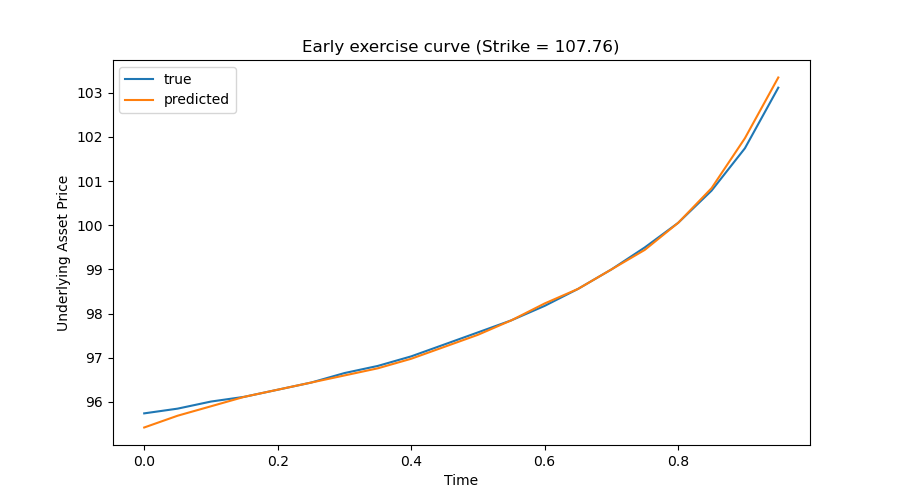}
\end{subfigure}\hfil 

\end{minipage}
\end{figure*}


\setcounter{figure}{0}
\begin{figure}[htb]
    \centering
    \begin{subfigure}[t]{0.48\textwidth}
        \includegraphics[width=\linewidth]{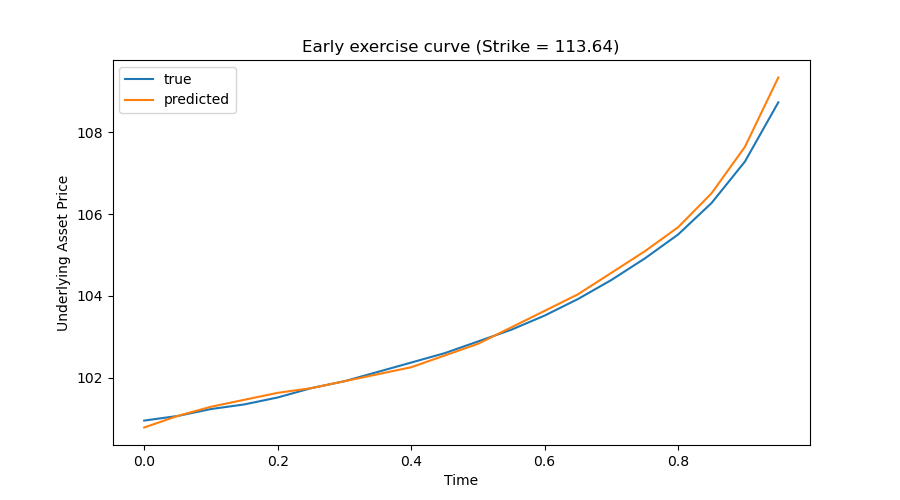}
        \label{fig:113}
    \end{subfigure}
    \hfill
    \begin{subfigure}[t]{0.48\textwidth}
        \includegraphics[width=\linewidth]{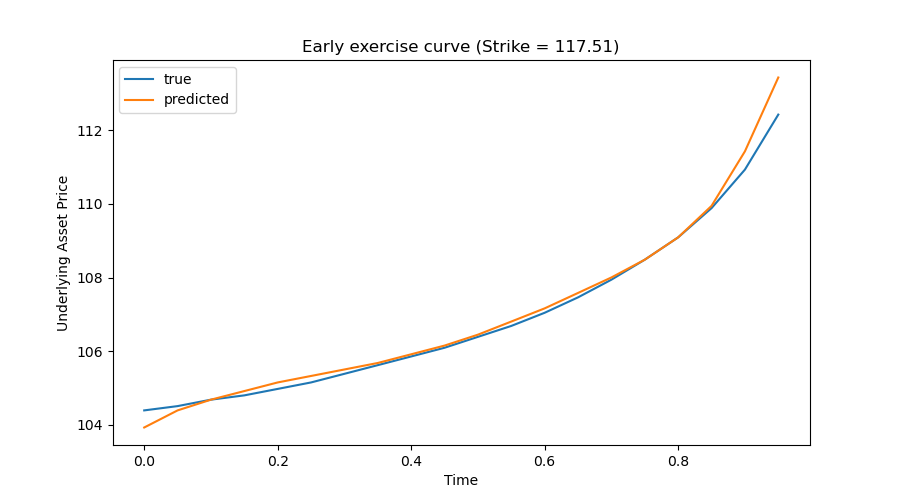}
        \label{fig:117}
    \end{subfigure}
    \caption{Exercise boundaries for American put options.}
    \label{fig:put}
\end{figure}

\newpage
 \IfFileExists{ref.bib}{%
   \bibliographystyle{plain}
   \bibliography{ref}
 }{%
   \input{references.bbl}
 }
\end{document}